\documentclass{article}

\usepackage[final]{neurips_2026}

\usepackage[utf8]{inputenc}
\usepackage[T1]{fontenc}
\usepackage{hyperref}
\usepackage{url}
\usepackage{booktabs}
\usepackage{amsmath,amssymb,amsfonts}
\usepackage{amsthm}
\usepackage{nicefrac}
\usepackage{microtype}
\usepackage{xcolor}
\usepackage{bm}
\usepackage{enumitem}
\usepackage{algorithm}
\usepackage{algpseudocode}
\usepackage{graphicx}
\usepackage{siunitx}
\usepackage{mathtools}
\usepackage{amsthm}
\newtheorem{theorem}{Theorem}
\newtheorem{proposition}{Proposition}
\newtheorem{remark}{Remark}

\usepackage{etoc}
\usepackage{wrapfig} 
\usepackage{float}

\title{Learning Compositional Latent Structure \\ with Vector Networks}
\author{
\begin{tabular}{c}
~\\[-0.2em]
\textbf{Niclas Pokel}\\
\textnormal{Institute of Neuroinformatics,}\\
\textnormal{UZH / ETH Zurich}\\
\textnormal{ETH AI Center}\\
\textnormal{Zurich, Switzerland}\\
\texttt{npokel@ethz.ch}
\end{tabular}
\and
\begin{tabular}{c}
~\\[-0.2em]
\textbf{Benjamin F.~Grewe}\\
\textnormal{Institute of Neuroinformatics,}\\
\textnormal{UZH / ETH Zurich}\\
\textnormal{ETH AI Center}\\
\textnormal{Zurich, Switzerland}\\
\texttt{bgrewe@ethz.ch}
\end{tabular}
}
\date{}

\newcommand{\R}{\mathbb{R}}
\newcommand{\vect}[1]{\bm{#1}}
\newcommand{\mat}[1]{\bm{#1}}
\newcommand{\ood}{\textsc{ood}}
\newcommand{\id}{\textsc{id}}

\begin{document}
\maketitle
\begin{abstract}
Deep networks are powerful function approximators, but they typically store many different computations in shared weight matrices, making it difficult to selectively reuse or adapt parts of them when a familiar structure appears in novel combinations. We introduce the \textbf{Vector Network} (VN), a hierarchical recurrent architecture in which each layer replaces a fixed weight matrix with a library of reusable rank-1 weight atoms. For each input, VN minimizes a layer-local energy to infer a sparse set of active weight atoms and their coefficients, jointly constrained by bottom-up input reconstruction and top-down feedback consistency. These weight atom coefficients then compose an input-specific low-rank weight matrix for that sample. After convergence, slow learning updates only the selected weight atoms through local residual signals scaled by the inferred coefficients. We evaluate VN on four compositional benchmarks spanning 1D signals, 2D spatial decoding, N-body dynamics, and compositional MNIST. VN matches strong baselines in distribution while often achieving out-of-distribution error about an order of magnitude lower when familiar factors must be recombined in novel ways. Vector networks thus make compositional generalization a structural property of the architecture and inference process rather than a brittle byproduct of fitting many behaviors into one shared dense parameter substrate.
\end{abstract}

\section{Introduction}
\label{sec:intro}

Modern deep networks trained with error backpropagation typically rely on a shared dense parameter substrate to implement many different computations. Because gradients are aggregated at the outputs and propagated through all layers, many different behaviors are written into the same shared weight space. Over training, this can produce what we term \emph{weight entanglement}: distinct features, skills, and behaviors become encoded in overlapping directions of the same weight matrix \cite{elhage2022toy}. Improvements for one behavior tend to move parameters along directions that also affect others. Since the objective constrains only the final fit, not how internal computations are decomposed, there is no direct pressure for weights to stabilize into identifiable, reusable functional components that could later be selectively recombined, although post-hoc interpretability methods can sometimes identify responsible human-aligned units or circuits \cite{bau2017network, olah2020zoom}.

This matters because, in a compositional system, one would like \emph{addressable responsibility in weight space}: learned weight components whose contributions can be isolated, recombined into a task-specific function, and updated independently. Standard deep networks do not naturally provide such persistent objects in parameter space. The problem is therefore not only that representations become mixed, but that the corresponding transformations are stored in one shared substrate rather than in reusable units. Several existing approaches partly address this problem. Routing-based models such as Mixture-of-Experts and capsules \cite{jacobs1991adaptive, shazeer2017sparsely, sabour2017dynamic, hinton2018matrix} introduce selective computation, but their selectivity operates through feedforward routing over dense modules or experts, whereas VN performs iterative sparse inference over a library of much smaller shared rank-1 weight components. Other approaches impose useful structure through post-hoc low-rank factorization, convolutional locality, or modular and object-centric inductive biases \cite{denil2013predicting, denton2014exploiting, lecun1998gradient, greff2019multi, locatello2020object}. These methods can be highly effective, but they still do not make reusable per-sample responsibility in weight space explicit. More generally, without appropriate inductive biases, latent factors are not typically recovered as reusable structure from observations alone \cite{locatello2019challenging}. Here we argue that if those factors are to be recombined to explain new inputs, the learned transformation must also itself become decomposable into reusable components.

Our \textbf{Vector Network} (VN) architecture thus takes a different approach (Fig.~\ref{fig:vector_scheme}A--C): it neither assigns predefined semantic roles to particular weight modules nor uses one fixed weight matrix. Instead, each layer is parameterized as a library of reusable rank-1 weight atoms (Fig.~\ref{fig:vector_scheme}A), each defined by a learned code vector pair as in Fig.~\ref{fig:vector_scheme}B, and for each input the model infers a sparse combination of these atoms to synthesize the effective transformation applied to that sample. \emph{Sparse Weight-Atom Inference} optimizes a scalar coefficient for each atom under sparsity, layer-input reconstruction, and top-down feedback consistency, so that only a small active set remains. After inference converged, the active weight atoms define the effective input-specific weight matrix (Fig.~\ref{fig:vector_scheme}C). The same inferred sparse responsibility pattern concentrates slow learning on that selected subset rather than writing every sample directly into one shared full-rank matrix. In this way, vector networks make compositionality an explicit weight-space inference problem, allowing learned weight components to be selectively recombined into input-specific transformations for novel combinations of familiar structure.

\begin{figure}
    \centering
    \begin{minipage}[t]{0.45\linewidth}
        \vspace{0pt}
        \centering
        \IfFileExists{Figures/fig1.png}{%
            \includegraphics[width=\linewidth]{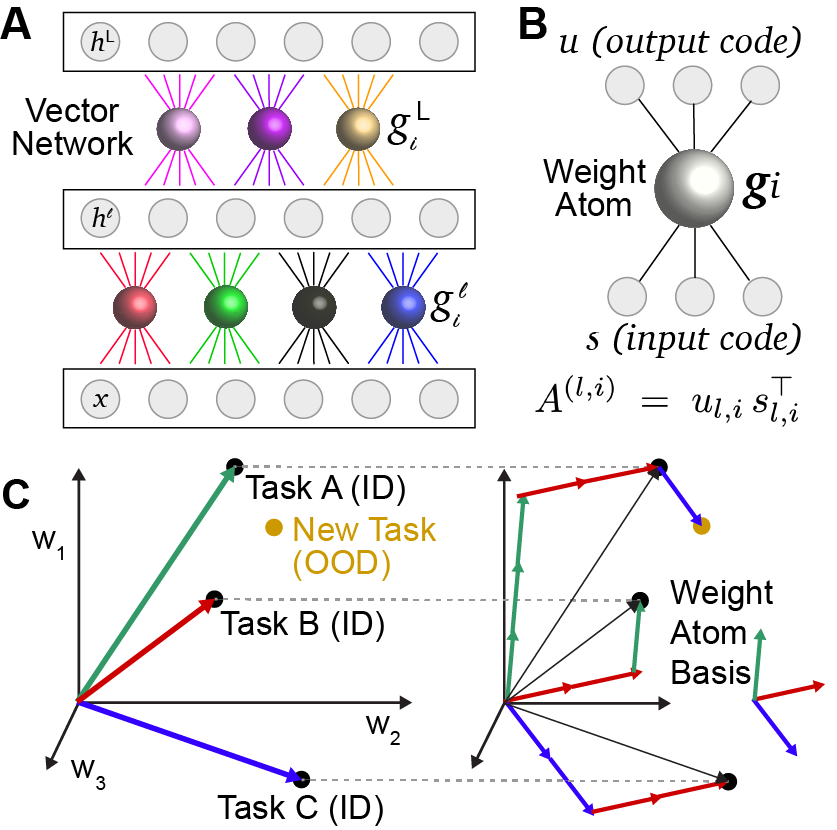}%
        }{%
            \fbox{\rule{0pt}{0.94\linewidth}\rule{0.94\linewidth}{0pt}}%
        }
    \end{minipage}\hfill
    \begin{minipage}[t]{0.55\linewidth}
        \vspace{0pt}
        \caption{\textbf{Vector Networks Compose Input-Specific Function from Sparse Weight Atoms.}
        \textbf{(A)} A VN layer replaces one dense weight matrix with reusable rank-1 weight atoms. Sparse coefficients \(g_{l,i}\) are inferred to reconstruct the current input \(x\) while propagating upward information to \(h_l\) and receiving reverse reconstruction targets from above.
        \textbf{(B)} A weight atom is a pair of code vectors defining \(A^{(l,i)}=u_{l,i}s_{l,i}^{\top}\), a rank-1 weight atom mapping sensing space \(s_l\) to output space \(u_l\). The activity \(g_{l,i}\) scales its contribution.
        \textbf{(C)} Standard networks apply a fixed monolithic matrix \(W\), so tasks share the same parameter space. VN instead composes a few fixed weight atoms into \(W_l^\star(x_l)=\sum_i g_{l,i}^\star(x_l)A^{(l,i)}\), allowing new inputs to be processed using recombinations of learned function components rather than interpolations within a fixed mapping.}
        \label{fig:vector_scheme}
    \end{minipage}
    \vspace{-10pt}
\end{figure}

\paragraph{Contributions.}
\begin{enumerate}[leftmargin=*]
\item \textbf{Rank-1 weight atoms as fundamental building blocks to compose network function.}
Each layer maintains a dictionary of reusable rank-1 weight atoms, allowing network function to be addressable, inspectable, and composable rather than entangled in a single dense weight matrix.

\item \textbf{The effective weight matrix is inferred by energy minimization over weight atom coefficients.}
For every sample a fast inference process minimizes a layer-local energy over weight atom coefficients under a sparsity constraint. Only after convergence do the inferred weight atoms define a tailored, input-specific \(W^\star(x)\).

\item \textbf{After inference, the selected weight atoms receive Atomic-Hebb updates.}
Slow learning follows the same energy locally, shaping the weight atom library through weight-atom-local residual updates scaled by inferred weight atom responsibility, so only selected atoms are modified.

\item \textbf{Sparse weight-atom responsibility structurally limits weight-update overlap.}
Because updates act only on the sparse active set, only a few responsible weight atoms are modified while inactive ones remain largely unchanged, limiting diffuse overlap between unrelated updates.

\end{enumerate}

\section{Related Work}
\label{sec:related}

\paragraph{Sparse coding, ICA, and dictionary learning.}
Classical sparse coding and ICA showed that observations can admit sparse, parts-based representations with explicit support selection \cite{olshausen1996emergence, olshausen1997sparse, bell1995information, bell1997independent, rozell2008sparse}. This gives a clear notion of responsibility, since only the active support explains the current input and is updated. VN inherits the sparse-inference idea, but shifts it from latent-state selection to inference over reusable \emph{weight atoms} that define the sample-specific transformation itself.

\paragraph{Predictive coding and energy-based inference.}
Predictive-coding and energy-based models also use iterative inference with local learning \cite{rao1999predictive, friston2010free, whittington2017approximation, oliviers2026bidirectional}. VN is closely related in optimization structure: here too, fast variables are latent coefficients inferred under fixed learned dictionaries. The distinction is architectural rather than procedural. In standard predictive coding or sparse coding, inferred activities are latent states within a fixed mapping. In VN, each inferred coefficient scales a learned rank-1 operator \(u_{l,i}s_{l,i}^{\top}\), so inference selects and composes the sample-specific operator itself rather than only the latent state passed through a fixed weight matrix. For a detailed comparison see Appendix~\ref{app:pc_vs_vn}.

\paragraph{Routing, modularity, and object-centric approaches.}
Mixture-of-experts, capsules, and object-centric models address compositionality by routing among experts, part-whole representations, or object slots \cite{jacobs1991adaptive, shazeer2017sparsely, sabour2017dynamic, hinton2018matrix, greff2019multi, locatello2020object}. These approaches can be highly effective, but their selectivity operates through feedforward routing over dense modules or latent entities, whereas VN uses iterative sparse inference under an explicit energy to select a competitive combination of shared rank-1 weight components that together define the sample-specific transformation.

\paragraph{Low-rank factorization and parameter-efficient adaptation (PEA).}
Low-rank factorizations and PEA methods constrain weights or updates to lower-dimensional subspaces \cite{denil2013predicting, denton2014exploiting, Hu2022LoRA}, but they do not usually provide explicit per-sample responsibility over reusable weight components.

\section{Vector Network Architecture}
\label{sec:vector-arch}

Next we define the core vector network architecture and weight-atom inference. Concretely, the modular structure is implemented by stacked layers of rank-1 weight atoms (Fig.~\ref{fig:vector_scheme}A). Each layer infers sparse weight atom coefficients for the current input, uses them to reconstruct the input and send a dense interface message upward, and receives a top-down target from the layer above during recurrent refinement. We index layers by \(l\in\{1,\dots,L\}\). Each layer has an input vector \(x_l\in\R^{d_l}\), a sparse signed activity vector \(g_l\in\R^{K_l}\), and a top-down interface target \(h_{l,\mathrm{target}}\in\R^{m_l}\). Each layer stores an input reconstruction dictionary
\(
S_l\in\R^{d_l\times K_l}
\)
and an interface dictionary
\(
U_l\in\R^{m_l\times K_l}
\),
with columns \(\{s_{l,i}\}\) and \(\{u_{l,i}\}\). Implementations may use separate interface dictionaries for forward and backward messaging; below we use the tied case \(U_l\) for readability. Dense coupling enforces \(d_{l+1}=m_l\) and identifies the next-layer input with the dense message,
\(
x_{l+1}\equiv \phi_l(U_l g_l).
\)

\subsection{Rank-1 Weight Atoms as Forward-Feedback Operators}
\label{sec:arch}

In layer \(l\), weight atom \(i\) is a paired set of unit code vectors in input and output space (Fig.~\ref{fig:vector_scheme}B; maintained at unit norm by the per-column renormalisation in Sec.~\ref{sec:slow}), which together define the rank-1 atom \(A^{(l,i)}\):
\begin{equation}
s_{l,i}\in\R^{d_l},\qquad u_{l,i}\in\R^{m_l},\qquad A^{(l,i)} = u_{l,i}s_{l,i}^{\top}.
\end{equation}
After inference has settled, the weight atom coefficients \(g_l^\star(x_l)\in\R^{K_l}\) synthesize an instance-conditioned weight matrix as a short sum of these weight atoms (Fig.~\ref{fig:vector_scheme}C),
\begin{equation}
W_l^\star(x_l) \;=\; \sum_{i=1}^{K_l} g^\star_{l,i}(x_l)\,A^{(l,i)} \;=\; \sum_{i=1}^{K_l} g^\star_{l,i}(x_l)\,u_{l,i}\,s_{l,i}^{\top} .
\label{eq:operator-synthesis}
\end{equation}
The same atom coefficients \(g_l\) serve two coupled roles: they reconstruct the current layer input as \(\hat x_l=S_l g_l\) and, through the layer output dictionary \(U_l\), generate the dense upward message \(h_l=\phi_l(U_l g_l)\) passed to the next layer, where \(\phi_l\) is typically identity or ReLU. Thus, inference does not apply one fixed dense matrix. Instead, it infers a sparse code \(g_l^\star\) that both selects a small set of reusable weight atoms and determines the upward message. We use \(W_l^\star(x_l)\) as a compact operator description induced by the inferred sparse code.
Top-down information enters each layer through \(h_{l,\mathrm{target}}\), a target later refreshed from the higher layer during recurrent inference (Sec.~\ref{sec:fast}). Because each \(A^{(l,i)}\) is rank 1, the rank of \(W_l^\star(x_l)\) is limited by how many weight atoms are active in \(g_l^\star\). Under an explicit hard top-\(k_l\) constraint, at most \(k_l\) atom coefficients are allowed to be nonzero, so the per-sample transformation has rank at most \(k_l\) regardless of the total dictionary size (see Proposition~\ref{prop:rank}, Appendix~\ref{app:theory:capacity}).
Because the reconstruction map is linear in the sparse code, superpositions are preserved at the decoder level: for two codes \(g_a\) and \(g_b\), \(S_l(g_a + g_b) = S_l g_a + S_l g_b\). Superposition is therefore preserved exactly, while exact recovery of novel compositions depends on the sparse inference conditions stated in App.~\ref{app:theory:composition}.

\subsection{Layer-Local Energy for Input Reconstruction, Atom Sparsity, and Top-Down Consistency}
\label{sec:layer_energy}

The deep VN tracks a layer-local energy in the sparse coefficients \(g_l\). For layer \(l\), we define the tied energy form as
\begin{equation}
\label{eq:layer_energy}
E_l(g_l \mid x_l, h_{l,\mathrm{target}})
=
\tfrac12\|x_l - S_l g_l\|_2^2
\;+\; \lambda_l \|g_l\|_1
\;+\; \tfrac12\|h_{l,\mathrm{target}} - U_l g_l\|_2^2 .
\end{equation}
Here \(\lambda_l\) is the layer-specific sparsity parameter, typically increased with depth so that higher layers use progressively sparser, more selective codes. The energy terms have the following roles: (i) \(S_l g_l\) reconstructs the current input \(x_l\), and (ii) \(U_l g_l\) remains consistent with \(h_{l,\mathrm{target}}\), the interface code supplied by deeper layers. At the top layer there is no higher layer to provide \(h_{L,\mathrm{target}}\), so the top-down term is omitted. The \(\ell_1\) penalty induces sparsity through the proximal soft-thresholding update
\(\mathrm{soft}(z,\tau)=\mathrm{sign}(z)\max(|z|-\tau,0)\),
setting coefficients below threshold to zero. Hidden layers use this soft-thresholding update, while the top layer is additionally constrained by a hard top-\(k_L\) support cap, meaning that at most \(k_L\) coefficients are allowed to remain nonzero.

For a single layer (fixed \(x_l,h_{l,\mathrm{target}},S_l,U_l\)), \(E_l\) acts as a Lyapunov energy for the fast coefficient inference: each accepted proximal step is required to decrease or leave unchanged this energy until a sparse fixed point is reached. In the full deep stack, the layer targets are themselves refreshed during recurrent inference, so the summed layer-local terms serve as per-sweep energy accounting rather than as a strict global Lyapunov proof. The same energy organizes both phases of the model: fast inference minimizes it with respect to the atom activities \(g_l\), while slow learning minimizes it with respect to the selected atom vectors. We next describe these two uses in turn.

\subsection{Fast Inference of a Sparse Set of Rank-1 Weight-Atom Coefficients}
\label{sec:fast}
Inference in the deep VN proceeds by a small number of global upward and downward sweeps. At each layer, coefficient inference is a proximal-gradient step on the layer energy \(E_l\), i.e.\ an ISTA-style update~\cite{parikh2014proximal}, with optional FISTA-style acceleration~\cite{beck2009fast}. For layer \(l\), we denote by \(\nabla f_l(g_l)\) the gradient of the smooth part of \(E_l\), i.e.\ the quadratic reconstruction and top-down consistency terms:
\begin{equation}
\nabla f_l(g_l) = S_l^\top(S_l g_l-x_l)+U_l^\top(U_l g_l-h_{l,\mathrm{target}}),
\end{equation}
where the \(U\) term is omitted when \(h_{l,\mathrm{target}}\) is not defined (top layer). With layer-specific step size \(\eta_l>0\), the soft-\(k\) update and its hard-\(k\) variant are
\begin{equation}
\label{eq:vn_update_soft}
g_l \leftarrow \mathrm{soft}\!\left(g_l-\eta_l\,\nabla f_l(g_l),\ \eta_l\lambda_l\right),
\end{equation}
\begin{equation}
\label{eq:vn_update_hard}
g_l \leftarrow \Pi_{k_l}\!\Big(\mathrm{soft}\!\big(g_l-\eta_l\,\nabla f_l(g_l),\ \eta_l\lambda_l\big)\Big).
\end{equation}
Eq.~\eqref{eq:vn_update_soft} is the standard proximal soft-thresholding step for the \(\ell_1\) term, so sub-threshold coefficients are set exactly to zero. Eq.~\eqref{eq:vn_update_hard} adds an explicit top-\(k_l\) projection \(\Pi_{k_l}\), which keeps the \(k_l\) largest-magnitude coefficients when a strict cardinality budget is required. In our experiments, hidden layers use the soft-\(k\) regime, while the top layer may use a strict hard top-\(k\) budget (see Remark~\ref{remark:topk} for the design rationale). At the network level, these local coefficient updates are organized into recurrent upward and downward passes, with the downward pass acting as a reverse reconstruction path in an encoder--decoder-like settle:
\begin{equation}
\label{eq:vn_sweeps}
\text{upward pass: } x_{l+1}\leftarrow h_l=\phi_l(U_l g_l),
\qquad
\text{downward pass: } h_{l,\mathrm{target}} \leftarrow \hat x_{l+1}=S_{l+1}g_{l+1}.
\end{equation}
In the upward pass, one coefficient update is taken at each layer as activity propagates bottom-up. In the downward pass, for each \(l=L{-}1,\dots,1\), the refreshed target is used to refine \(g_l\) again through the same update rule, now using both bottom-up and top-down residuals; the top-down term plays the role of a predictive-coding consistency signal, which we discuss further in App.~\ref{app:theory:topdown}. Inference stops when the relative energy improvement falls below a tolerance \(\tau_E\) or after a maximum number of sweeps; a sweep may also be rejected if it increases energy. With \(\eta_l\) scaled by the inverse Lipschitz constant of the local Hessian, the soft-\(k\) proximal updates monotonically decrease the layer energy and converge to a stationary point (Theorem~\ref{thm:convergence}, Appendix~\ref{app:theory:convergence}).

\subsection{Slow Learning Updates the Inferred Rank-1 Weight Atoms}
\label{sec:slow}

Once fast inference has converged, learning updates only the selected atom vectors, i.e.\ the dictionary columns with nonzero coefficients, using strictly local residual\(\times\)atom coefficient signals. The inactive atoms are left unchanged, and inference itself is treated as a black box that is \emph{not} differentiated through.
For layer \(l\), we define the settled residual errors in the two spaces as
\begin{equation}
r_l^x = x_l - S_l g_l^\star,\qquad r_l^h = h_{l,\mathrm{target}} - U_l g_l^\star
\end{equation}
(when \(h_{l,\mathrm{target}}\) is defined). A direct local gradient step is then
\begin{equation}
S_l \leftarrow S_l + \rho^S_l\,r_l^x\,(g_l^\star)^\top,
\qquad
U_l \leftarrow U_l + \rho^U_l\,r_l^h\,(g_l^\star)^\top,
\end{equation}
where \(r\,g^\top\) denotes the batch-averaged outer product over a batch of size \(B\), and \(\rho^S_l,\rho^U_l\) are local learning rates. For the top layer (no \(h_{L,\mathrm{target}}\)), the interface update is skipped. In the implementation, these residual\(\times\)atom-coefficient terms are used as local gradients for Adam, followed by renormalization. We term this update rule \emph{Atomic-Hebb} because it is distinct from classical pre--post(error) Hebbian type plasticity: each step combines local residual error with inferred weight-atom contribution, so credit assignment is weight-atom local rather than neuron local. Further, by applying Danskin's Theorem at the inference equilibrium, we prove that these local residual products are exactly equivalent to the implicit gradients of the global free energy (see Theorem~\ref{thm:danskin}, Appendix~\ref{app:theory:gradients}).

\section{Experiments}
\label{sec:experiments}

Across the following four experiments, models are trained in distribution (ID) and evaluated on both ID and out-of-distribution (OOD) cases formed by novel combinations of previously seen factors. OOD cases span increasing compositional difficulty, from simple recombinations to harder interactions of familiar factors. Because most are reconstruction tasks, we report absolute reconstruction error and how performance changes from familiar ID structure to composed OOD cases.

\subsection{Spatial Composition: Generalizing to Novel Bump Positions}
\label{sec:blob}

We begin with a simple spatial reconstruction benchmark adapted from Liang et al.~\cite{fiete2025}, because it isolates the weight entanglement problem in a particularly clean form. As shown in Fig.~\ref{fig:blob_fiat}a, the model receives a low-dimensional code for the \(x\)- and \(y\)-position of a target bump and must reconstruct a Gaussian bump at that location on a \(28\times 28\) grid. Training is restricted to a local spatial region outside a square holdout zone, while OOD evaluation probes that unseen center region. This makes the benchmark a minimal test of whether a learned decoder can preserve a reusable spatial factorization off support, rather than only interpolating within the training region.

Our reproduction recovers the basic pattern reported by Liang et al. \cite{fiete2025}. The spatial error maps in Fig.~\ref{fig:blob_fiat}b show that standard CNN- and VAE-based decoders develop pronounced OOD error inside the holdout zone, whereas hand-factorized decoders remain much more structured but less performant. Quantitatively, Fig.~\ref{fig:blob_fiat}c shows the core tradeoff: the CNN exhibits a large OOD decline, whereas the hand-built rank-1 factorized decoder generalizes more predictably but remains limited in absolute reconstruction performance. VN improves on both, combining low OOD error with strong ID performance. Figure~\ref{fig:blob_fiat}d strengthens this point by showing that the OOD/ID ratio of the CNN decoder worsens with depth, i.e.\ the deeper the network, the more strongly entangled its off-support input--output mapping becomes.

VN seems to combine the two desiderata more effectively. A VN decoder maintains uniformly low reconstruction error across the full spatial domain (Fig.~\ref{fig:blob_fiat}b). It achieves substantially lower OOD error while preserving strongest ID performance (Fig.~\ref{fig:blob_fiat}c), and even improves slightly with network depth (Fig.~\ref{fig:blob_fiat}d). The full VN encoder--decoder improves robustness further by discovering the factorization end to end rather than receiving it as a hand-built prior. Thus, even in this simple toy setting, VN does not merely avoid CNN weight entanglement; it learns some form of a reusable decomposition in its weight atoms that supports reconstruction in the unseen spatial region. For extended results on this experiment please see Appendix~\ref{app:fiete_extended}; full sweeps and component-removal ablations are reported in Apps.~\ref{app:fiete_vn_ablations}--\ref{app:entangled}.

\begin{figure}[H]
    \noindent
    \begin{minipage}[t]{0.5\linewidth}
        \vspace{0pt}
        \IfFileExists{Figures/fig2.png}{%
            \includegraphics[width=\linewidth]{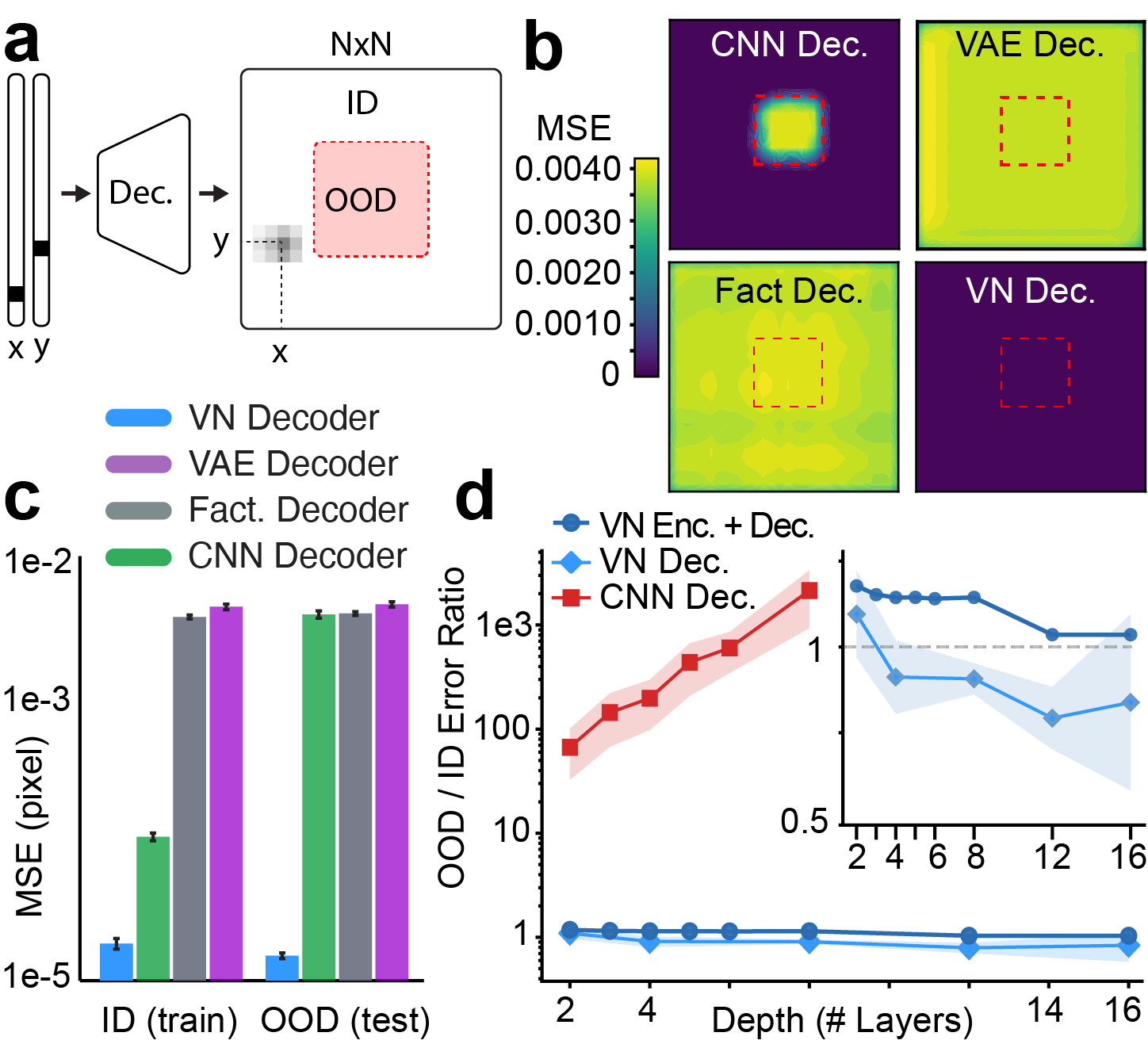}%
        }{%
            \fbox{\rule{0pt}{0.45\textheight}\rule{0.95\linewidth}{0pt}}%
        }
    \end{minipage}\hfill
    \begin{minipage}[t]{0.48\linewidth}
        \vspace{0pt}
        \caption{
        \textbf{Spatial Composition: Generalizing to Novel Bump Positions}
        \textbf{(a)} Task setup (adapted from \cite{fiete2025}). A decoder maps \((x,y)\) to an \(N\times N\) output grid. Training covers a local ID region, while testing probes an unseen OOD region.
        \textbf{(b)} Spatial reconstruction error (MSE) across the grid. CNN, VAE, and factorized decoders show strong OOD error, whereas the VN decoder remains low across the full space.
        \textbf{(c)} Reconstruction MSE on ID and OOD regions. VN-based models achieve much lower OOD error while maintaining strong ID performance. Error bars are 95\% CI.
        \textbf{(d)} OOD/ID error ratio vs.\ depth. CNN decoders degrade rapidly with depth, whereas VN remains stable and improves with depth (inset: zoom on VNs).
        }
        \label{fig:blob_fiat}
    \end{minipage}
    \vspace{-8pt}
\end{figure}

\subsection{Function Composition: Generalizing to Unseen Signal Combinations}
\label{sec:functions}

We next test the compositional capabilities of VN by training on single primitive functions and evaluating on unseen compositions of these primitives (Fig.~\ref{fig:fc_overview}a). Each 1D signal is sampled from one of four elementary function families---\(a\sin(ft+\phi)\), \(a\cos(ft+\phi)\), a bounded polynomial \(a(t/2\pi)^f\), or a Gaussian bump \(a\,e^{-0.3f(t-\pi)^2}\)---with randomized parameters. The ID task is to reconstruct signals from a single family, while the OOD sets require novel compositions of those same primitives: Easy-OOD uses cross-family sums such as \(\sin(f_1 t+\phi_1)+a\cos(f_2 t+\phi_2)\), whereas Hard-OOD uses nested and mixed nonlinear compositions such as \(\sin\!\bigl(\cos(f_1 t)+f_2 t/(2\pi)\bigr)+a\cos\!\bigl(f_3(t/(2\pi))^2\bigr)\). This first part of the benchmark is therefore a reconstruction task, asking whether a model can reuse familiar components when the signal structure changes beyond the training distribution.

\begin{figure}[h!]
    \centering
    \includegraphics[width=1\linewidth]{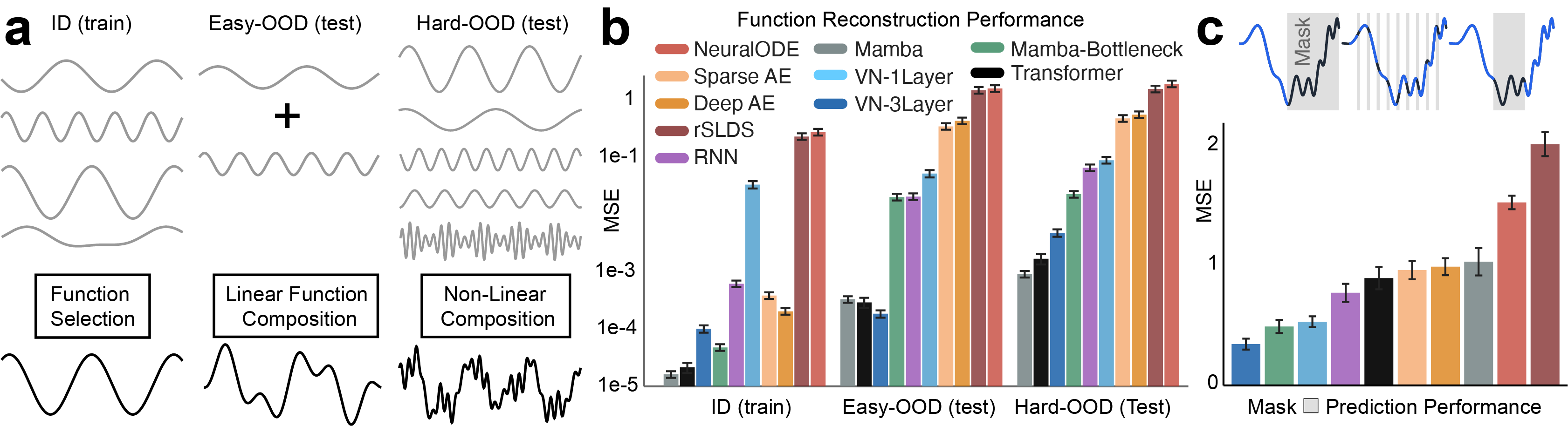}%
    \caption{
    \textbf{Function Composition: Generalizing to Novel Signal Combinations. }
    \textbf{(a)} Training and test structure. Models are trained on individual function families drawn from the in-distribution (ID) set. Easy-OOD tests require linear recombination of learned functions, whereas Hard-OOD tests require non-linear compositions of previously learned functions.
    \textbf{(b)} Reconstruction MSE for ID, Easy-OOD, and Hard-OOD conditions.
    \textbf{(c)} Average Hard-OOD masked prediction MSE across the four masking regimes shown above. All bar plots are sorted by mean performance; error bars are 95\% CI. VN-1L/3L: 1 and 3 layer VN.
    }
    \label{fig:fc_overview}
    \vspace{-8pt}
\end{figure}

The reconstruction results in Fig.~\ref{fig:fc_overview}b show that the benchmark is nontrivial, but they do not yet cleanly separate compositional structure learning from strong sequence modeling. In particular, both Transformer and Mamba outperform VN on the Easy-OOD and Hard-OOD reconstruction settings. At the same time, the bottlenecked Mamba variant performs worse than Mamba on ID but substantially better on OOD, suggesting that constraining the representation already pushes the model toward more reusable internal structure.

This ambiguity arises because reconstruction still admits shortcut solutions: when much of the relevant signal is directly present in the input, a strong sequence model can achieve low error by preserving or continuing that structure without recovering components that can be recombined systematically. VN, by contrast, is intentionally sparsity constrained and must explain each sample through a small active set of weight atoms. Reconstruction alone is therefore informative, but not yet a clean test of genuinely compositional internal structure.

To resolve this ambiguity, we introduce the masked prediction variant in Fig.~\ref{fig:fc_overview}c, where part of the signal must be forecast or filled in from the learned representation. This removes near-identity shortcut solutions and makes the task a stricter test of whether the model has learned reusable structure. Under this evaluation, the baseline models degrade much more strongly, whereas VN remains substantially more stable and achieves the clearest advantage. Taken together, these results suggest that VN's strength in this benchmark is not simply smoother reconstruction, but the recovery of sparse reusable latent structure that supports both reconstruction and prediction under novel function compositions. For extended results and performance analysis please see Appendix~\ref{app:fc}.

\subsection{Force Composition: Generalization in N-body Dynamics}
\label{sec:gravity}

We next turn to force composition in n-body dynamics, a substantially harder physics benchmark in which the relevant latent factors are the underlying force laws themselves (Fig.~\ref{fig:gravity_comp}a). Models are trained only on single-force systems (gravity, drag, Lorentz, or spring) and are then evaluated by predicting per-body accelerations under two simultaneous OOD shifts: multiple forces acting together and reduced body count. Thus the ID setting consists of single-force dynamics with \(n=5\) bodies, whereas the OOD setting requires reconstructing mixed-force systems with \(n=5,4,3\) bodies. This is a much stronger compositional test than interpolation alone: although instantaneous forces superpose linearly, the resulting trajectories diverge under coupled state feedback through positions and velocities, so the model must reuse learned physical mechanisms in unseen superpositions and under shifted state distributions.

\begin{figure}[h!]
    \centering
    \IfFileExists{Figures/fig4.png}{%
        \includegraphics[width=\textwidth]{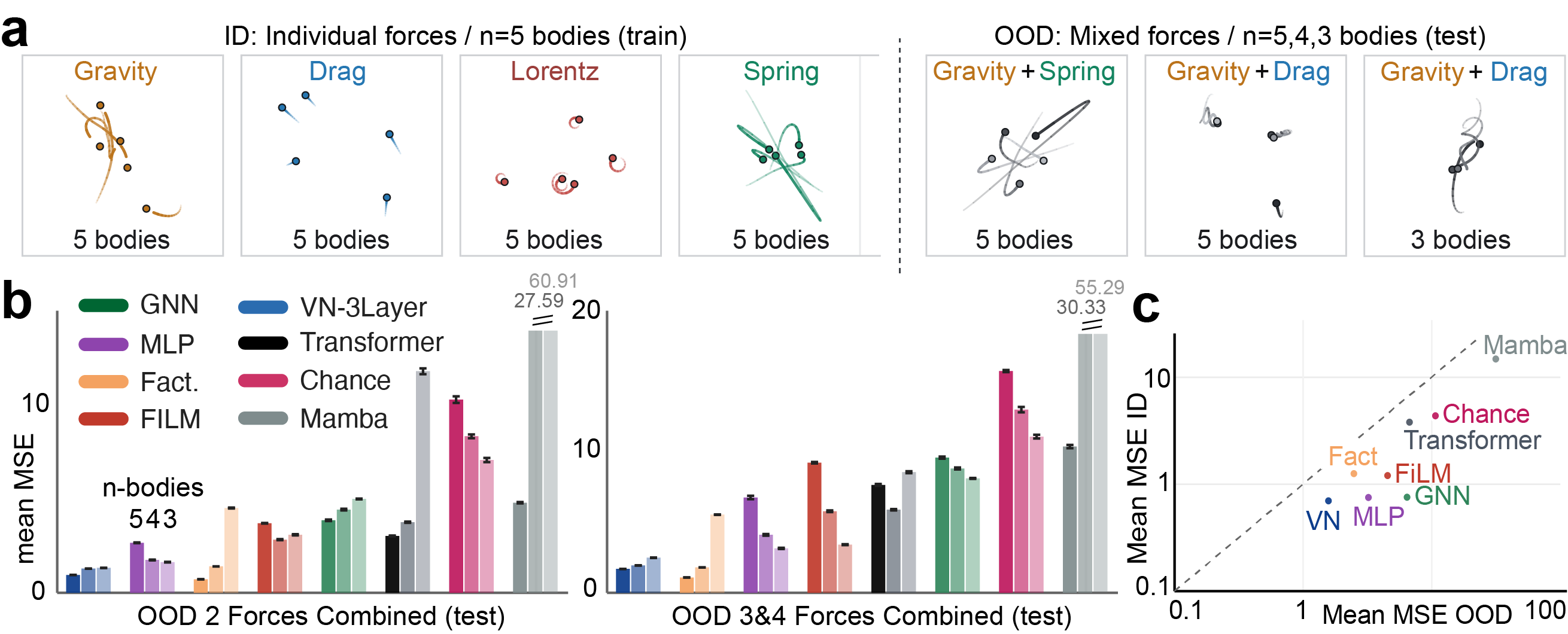}%
    }{%
        \fbox{\rule{0pt}{0.42\textheight}\rule{0.98\textwidth}{0pt}}%
    }
    \caption{
    \textbf{Force Composition: Generalization in n-body Dynamics.}
    \textbf{(a)} Task setup. Models are trained on single force fields (gravity, drag, Lorentz, spring) with \(n=5\) bodies and then evaluated out of distribution on mixed-force systems and reduced body counts (\(n=5,4,3\)).
    \textbf{(b)} Mean OOD MSE under force composition. Left: mean over all two-force combinations. Right: mean over all three- and four-force combinations. Error bars are 95\% CIs; within each model, lighter bars denote reduced body count (\(n=4,3\)). 
    \textbf{(c)} Mean OOD vs.\ mean ID performance on a log--log scale, averaged across \(n=5,4,3\). Each point is one model, with ID MSE on the y-axis and OOD MSE on the x-axis. The dashed diagonal marks equal ID and OOD performance; models that degrade under distribution shift move to the right and below this line. VN remains closest to the ideal low-error, near-diagonal regime, indicating stronger robustness and compositional generalization.
    }
    \label{fig:gravity_comp}
\end{figure}

The main quantitative result is shown in Fig.~\ref{fig:gravity_comp}b. All models perform well on the single-force ID cases, but errors rise substantially once forces are combined, and rise further as evaluation moves from two-force mixtures to the harder three- and four-force regimes. The effect becomes even stronger under body dropout, shown by the lighter \(n=4\) and \(n=3\) bars. In contrast, the 3-layer VN remains consistently low across these OOD settings and retains a clear margin over MLP, GNN, Transformer, Mamba, and FiLM baselines. We also compare against a 'Factorized' control, which uses one hand-crafted subnetwork per force and combines their fitted contributions at test time, as well as a random baseline corresponding to chance-level prediction. Even against these controls, fitting the primitive regimes in isolation is not enough: what matters is whether the learned mechanisms can still be recombined when the physical system changes. For extended results see Appendix~\ref{app:nbody_extended}.

Figure~\ref{fig:gravity_comp}c makes this interpretation explicit by comparing mean ID and mean OOD error directly on a log--log scale. Several baselines achieve low ID error yet move far to the right under compositional shift, indicating that good single-force reconstruction does not imply robust transfer to mixed-force dynamics. VN remains closest to the low-error, near-diagonal regime, showing that its advantage here is not simply stronger ID fitting, but more stable reuse of learned force-law structure under increasing compositional load. In that sense, the n-body benchmark provides the clearest evidence in our suite that explicit weight-space responsibility can support mechanism-level composition in a difficult nonlinear setting. 

\subsection{Image Composition: Generalizing to Novel MNIST Digit Compositions}
\label{sec:mnist}

To test whether the same VN capability transfers to vision, we evaluate two compositional MNIST benchmarks constructed for this study (Fig.~\ref{fig:mnist_rts}a--c). The first is a digit reconstruction benchmark in which the ID setting consists of single digits, while the OOD conditions require compositions such as Add/Subtract mixtures, Max projections over multiple digits, and structured corruptions; Appendix~\ref{app:mnist_extended} reports additional representation and convergence diagnostics. The second is a dynamic Rotation/Translation/Scaling (RTS) benchmark in which single transformations define the ID setting and OOD evaluation requires stronger and more complex transformation compositions, including larger step sizes and multiple digits transformed simultaneously (Fig.~\ref{fig:mnist_rts}b). In both cases the task is reconstruction: for standard MNIST, the model reconstructs the image itself, whereas for RTS it reconstructs the transformation-induced motion signal \(dx\). Baselines are compared against autoencoders (AEs) at matched ID performance or matched bottleneck sparsity, so any OOD differences cannot be explained by capacity or sparsity alone.

\begin{figure*}[h!]
    \centering
    \IfFileExists{Figures/fig5.png}{%
        \includegraphics[width=1\linewidth]{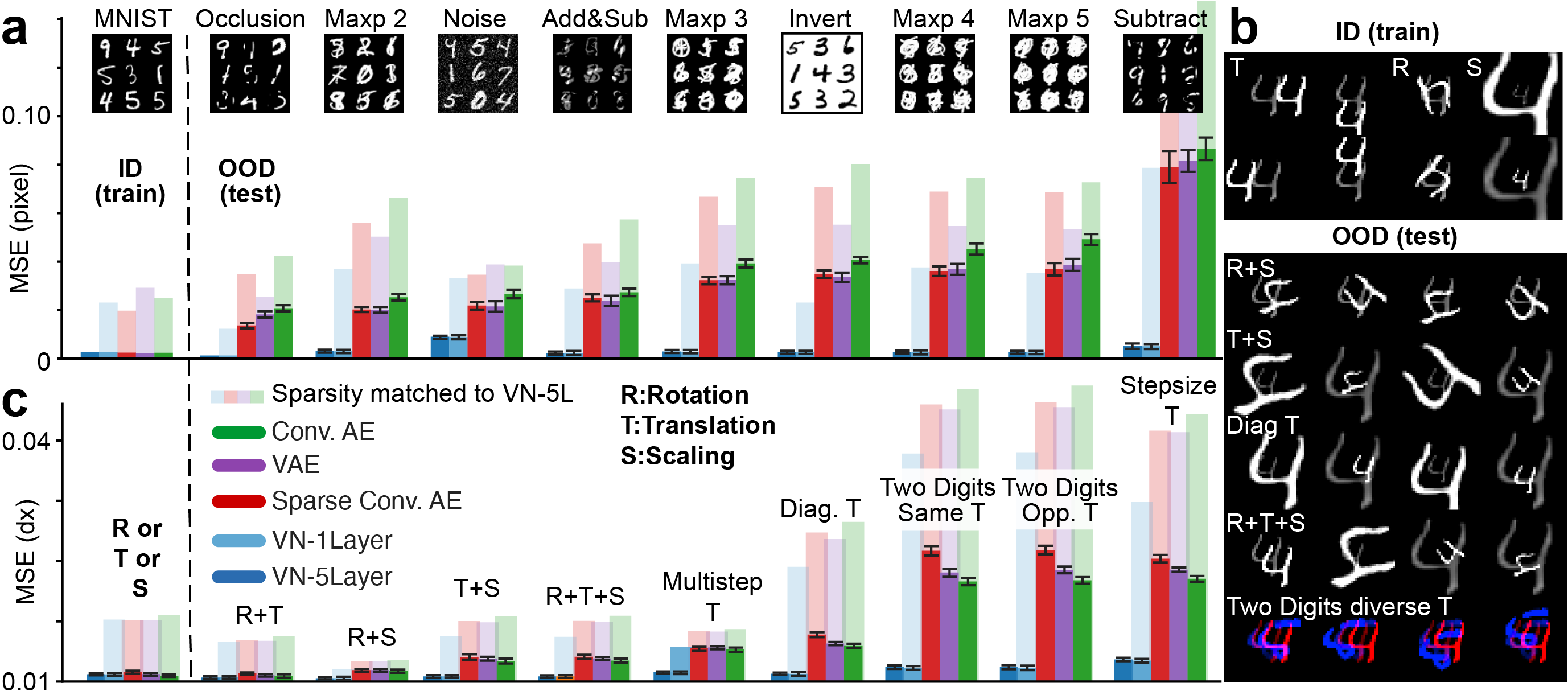}%
    }{%
        \fbox{\rule{0pt}{0.45\textheight}\rule{0.98\textwidth}{0pt}}%
    }
    \caption{\textbf{Image Composition: Generalizing to MNIST Compositions}
    \textbf{(a)} VN reconstruction error on MNIST ID and OOD composition tests.
    \textbf{(b)} Representative Rotation/Translation/Scaling (RTS) ID training samples and OOD test compositions of different motions.
    \textbf{(c)} Motion reconstruction error on ID/OOD RTS transformations; faint bars denote sparsity-matched AE variants. OOD tests combine familiar transformations outside the training distribution, and VN again achieves the strongest OOD performance while remaining competitive on ID. Error bars are 95\% CI.}
    \label{fig:mnist_rts}

\end{figure*}

The digit benchmark already shows a clear separation (Fig.~\ref{fig:mnist_rts}a). All models perform well on single-digit ID cases, but the difference appears once the test set requires composition or corruption. Under Add/Subtract mixtures, Max projections, and structured corruptions, AE reconstruction errors increase substantially, whereas the VN remains much closer to its ID performance. In the strongest OOD conditions, VNs are roughly an order of magnitude better than the corresponding AE baselines. Sparsity-matched AEs also perform worse across both ID and OOD tests, showing that the advantage is not sparsity alone but the VN's sparse weight-atom structure and reuse.

The RTS benchmark shows the same pattern for visual transformations rather than digit identity (Fig.~\ref{fig:mnist_rts}b--c). For single transformations, all models perform similarly. When the OOD setting requires larger transformation steps or simultaneous transformations of multiple digits, AE baseline errors increase steadily, while the VN degrades much more gracefully and preserves digit structure more reliably. Thus the same advantage that appears in static digit composition also extends to a dynamic visual setting in which the relevant factors are transformations rather than image content. 

Overall, the MNIST results provide the strongest visual confirmation of the central pattern seen throughout the paper. Strong performance on simple ID regimes does not imply compositional generalization, even in vision. Dense baselines perform well when interpolation is sufficient, but degrade once familiar components or transformations must be recombined in unfamiliar ways. The VN remains substantially more stable across both content and motion settings, consistent with the idea that explicit weight-space responsibility can support reusable structure across the spatial, function-composition, and physics benchmarks considered above.

\section{Discussion}
\label{sec:discussion}

While VN is often comparable to strong baselines on in-distribution cases, its advantage becomes clear in out-of-distribution settings where previously learned components must be recombined into novel compositions. In those regimes, VN frequently achieves substantially lower error, often by about an order of magnitude, and often achieves lower error than strong modern baselines such as Transformers and Mamba. The main advantage of VN therefore seems to lie less in ordinary interpolation than in preserving reusable structure as compositional load increases.

The architecture and learning rule of VN explicitly target the preservation of reusable structure under increasing superposition by making weight responsibility explicit before plasticity. VN first infers which weight atoms were responsible for the current explanation and then updates those components through atom responsibility\(\times\)error (Atomic-Hebb). This fast-recombination / slow-storage split localizes plasticity and helps reduce diffuse overlap between unrelated updates while preserving weight-atom recomposability across inputs (see Remark~\ref{remark:hebbian_vs_bp} for contrast to BP credit assignment).

The same weight-atom structure suggests several broader advantages. First, once a useful library of latent factors has been learned, new inputs can be handled by recomposition rather than by storing each factorial combination separately. This offers a natural route toward improved data efficiency, especially in the long tail, where the challenge is not simply interpolation between seen cases but reuse of familiar parts in unfamiliar combinations. In this sense, VN may scale more gracefully with the number of latent factors that must be discovered than with the number of training examples needed to cover their combinations. Explicit weight-atom responsibility may also be useful for continual learning, because it provides a direct sample-wise signal of which reusable components were important for the current explanation. A cautious extension is that some new behavior may be incorporated by refining a small subset of existing components, or where necessary adding a few new ones, rather than reworking an entire dense shared substrate. In spirit this is related to parameter-efficient schemes and their variants such as LoRA, which also restrict adaptation to a low-rank subspace, but here the relevant components are selected and later complemented by input-dependent inference rather than introduced as fixed adaptation factors \cite{Hu2022LoRA}.

The main limitation of the current VN is practical rather than conceptual: explicit weight-atom responsibility is presently obtained through a recurrent sparse settle. While this yields interpretability and modularity, in its current form the VN inference is significantly slower than a single feedforward application and currently limits efficiency. This also points to the most immediate path forward. Prior work has shown that iterative sparse-coding inference can be approximated by learned feedforward updates \cite{gregor2010learning}, suggesting that learned support initializers or hybrid fast/slow inference could substantially shorten the VN settle once the sparse support has stabilized \cite{graves2016adaptive, marino2018iterative, monga2019algorithm}. 

Overall, the contribution of the VN is not only in the current benchmarks, but in the design principle it makes explicit: generalization should rely on reusable weight structure rather than on continually fitting new cases into one monolithic parameter field. In this sense, VN points beyond standard backpropagation toward a regime in which computation first infers sparse responsibility over reusable weight components and learning then acts selectively on that inferred structure. Conceptually, VN shifts compositionality from state routing to responsibility in weight space, which may matter most in real-world settings where familiar factors must be reused in unfamiliar combinations.

\section*{Acknowledgements}
We thank Sander de Haan, Giuseppe Chindemi and Yassine Taoudi-Benchekroun for comments on the manuscript. This work was supported by the Swiss National Science Foundation (189251, 10003772, B.F.G.) and ETH Zurich project funding (ETH-20 19-01, B.F.G.).

\section*{Author Contributions}
B.F.G. conceived the Vector Network architecture, Atomic Hebb rule and developed the first multilayer implementation. N.P. and B.F.G. jointly developed and wrote the present manuscript and the OOD experimental settings. N.P. carried out the mathematical proofs, the function-composition, spatial-decoding (blob), and n-body experiments, and contributed substantially to the experimental analysis and presentation. B.F.G. and N.P. jointly refined the implementation and the final experimental framework.

\bibliographystyle{plainnat}
\bibliography{references}

\clearpage
\appendix

\section*{Appendix}
\section*{Content List}
\begin{itemize}[leftmargin=1.5em]
  \item Predictive Coding vs.\ Vector Networks (Sec.~\ref{app:pc_vs_vn})
  \item VN vs.\ Transformers: Sparse Operator Inference vs.\ State Routing (Sec.~\ref{app:vn_vs_transformers})
  \item Fast Inference of Weight Atom Coefficients in Deep Vector Networks (Sec.~\ref{app:fast_inference})
  \item Slow Local Dictionary Updates: The Atomic-Hebb rule (Sec.~\ref{app:atomic_hebb})
  \item Theoretical Foundations and Proofs (Sec.~\ref{app:theory})
  \item Spatial Decoding: Extended Results (Sec.~\ref{app:fiete_extended})
  \item Function Composition: Extended Results (Sec.~\ref{app:fc})
  \item n-Body Dynamics: Extended Results (Sec.~\ref{app:nbody_extended})
  \item MNIST Composition: Extended Results (Sec.~\ref{app:mnist_extended})
  \item Reproducibility: Implementation Details (Sec.~\ref{app:impl})
  \item Hardware and Storage Requirements (Sec.~\ref{app:hardware})
\end{itemize}

\section{Predictive Coding vs.\ Vector Networks}
\label{app:pc_vs_vn}

This section clarifies the precise relation between predictive coding (PC) and Vector Networks (VN). The goal is to separate the shared optimization structure from the different architectural role played by the inferred variables. 

The comparison between predictive coding (PC) and Vector Networks (VN) is easiest to understand by separating \emph{how} inference is performed from \emph{what role} the inferred variables play. In both cases, fast variables are inferred under fixed learned parameters by minimizing a local objective or energy \cite{rao1999predictive, friston2010free, whittington2017approximation, oliviers2026bidirectional}. The key difference is architectural: in PC, the inferred variables are the dense layer activities themselves, whereas in VN the inferred variables reparameterize the transformation between dense layers.

In standard sparse or predictive coding, the mapping has the form
\begin{equation}
y = W x.
\end{equation}
Writing \(W=[w_1,\dots,w_d]\) in terms of its columns gives
\begin{equation}
y = \sum_j x_j\, w_j.
\end{equation}
Thus, changing the inferred activity \(x_j\) gates the contribution of column \(w_j\). Inference therefore selects which input directions or features are active, but the output-side matrix \(W\) remains fixed. Put differently, PC infers the dense hidden state and only thereby determines how the fixed matrix is used; the inferred object is the layer activity itself, not an explicit parameterization of the transformation between layers.

In VN, by contrast, the effective operator is synthesized from reusable rank-1 atoms:
\begin{equation}
W^\star(x) = \sum_i g_i(x)\,u_i s_i^\top,
\end{equation}
so the inter-layer transformation is first reparameterized by the sparse coefficient vector \(g\), and inference then minimizes the layer energy with respect to that code. The coefficients \(g_i\) can be viewed as latent activities, but they are more than that: each \(g_i\) scales an entire rank-1 operator \(u_i s_i^\top\), so the inferred code defines the sample-specific operator description between the two dense layers. Once \(g_l^\star\) is inferred, the dense upward activity \(h_l\) (identified with \(x_{l+1}\)) is produced directly from that settled code rather than being itself inferred as an independent dense hidden state. In this sense, VN does not merely select latent features under a fixed map, but selects and composes input--output mappings, i.e.\ operators.

The distinction can be summarized as follows:
\begin{equation}
\text{PC: } y=\sum_j x_j\,w_j
\qquad
(\text{column / feature selection}),
\end{equation}
\begin{equation}
\text{VN: } W^\star(x)=\sum_i g_i\,u_i s_i^\top
\qquad
(\text{operator selection}).
\end{equation}
Thus both frameworks use iterative latent inference, but the inferred variables control different objects: PC infers dense states within a fixed mapping, whereas VN infers a sparse code that assembles the sample-specific operator itself and then yields the dense next-layer activity.

The same shift can be understood more concretely from a bidirectional predictive-coding (bPC) stack. For a mapping \(x_l \xrightarrow{W_l} x_{l+1}\), both neighboring activities \(x_l\) and \(x_{l+1}\) are inferred in bPC. The effective use of the fixed matrix \(W_l\) is therefore jointly shaped from below by the inferred lower-layer state and from above by the inferred upper-layer state. In VN, by contrast, the dense intermediate activities are fixed-pass communication channels rather than inference variables; the model alternates dense communication layers \(x_l\) with sparse inference layers \(g_l\), so information flows as \(x_l \rightarrow g_l \rightarrow x_{l+1}\). This removes the shared-state coupling of bPC and assigns the layer-local transformation to one explicit sparse code.

Second, the fixed dense matrix \(W_l\) is replaced by a layer-local dictionary of rank-1 operators \(u_{l,i}s_{l,i}^\top\), and the sample-specific transformation is assembled as
\begin{equation}
W_l^\star(x_l)=\sum_i g_{l,i}(x_l)\,u_{l,i}s_{l,i}^\top.
\end{equation}
Inference therefore no longer selects only neural activities under fixed weights, but selects and composes reusable operators. Crucially, this operator selection is explicitly sparse: the coefficients \(g_l\) are inferred under an \(\ell_1\) penalty, and at the top layer may additionally be subject to a hard top-\(k\) budget. Only a small competitive subset of atoms therefore remains active, and a single inferred coefficient vector \(g_l\) defines the operator description between the two dense layers, rather than that transformation emerging indirectly from the interaction of two inferred neighboring states. Third, the learning rule becomes weight atom local (not neuron local):
\begin{equation}
\Delta S_l \propto r_l^x g_l^\top,
\qquad
\Delta U_l \propto r_l^h g_l^\top.
\end{equation}
In bPC, local plasticity typically takes a pre\(\times\)error or pre\(\times\)post\(\times\)error form, and the inferred activities are usually not constrained to select a sparse competitive subset of reusable operators, so updates remain distributed across adjacent weight matrices and are neuron- or pathway-local rather than weight-atom local. This distributed coupling leads to stronger weight entanglement across layers, because weight responsibility remains spread over neighboring inferred states and shared dense matrices. In VN, the inferred coefficients \(g_l\) update only the atoms of their own layer, so credit assignment becomes weight-atom local to the transformation itself and more directly limits such cross-layer entanglement. In this sense, the key conceptual shift is
\begin{equation}
\text{bPC: inference on states}
\qquad\longrightarrow\qquad
\text{VN: inference on transformations / operators.}
\end{equation}

Bidirectional predictive coding therefore remains a state-inference framework: bottom-up and top-down activities are inferred under fixed forward and backward weights, and plasticity stays neuron- or pathway-local through pre\(\times\)error or pre\(\times\)post\(\times\)error interactions. VN keeps the iterative inference structure but shifts both computation and learning onto sparse reusable operator atoms: the inferred coefficient \(g_l\) selects a competitive subset of rank-1 transformations and the same coefficient gates their weight-atom-local update. This is the sense in which VN turns compositionality into explicit responsibility over reusable weight structure.

\section{VN vs.\ Transformers: Sparse Operator Inference vs.\ State Routing}
\label{app:vn_vs_transformers}

This section compares VN to transformer attention at the level of the adaptive variables each model uses during inference. The central distinction is between feedforward state routing in transformers and sparse operator inference in VN.

The comparison to transformers is useful because both architectures are input-adaptive, but the adaptive variables act on different objects. In a transformer, attention weights route information among token states:
\begin{equation}
\vect h_i^{\mathrm{attn}}=\sum_{j}\alpha_{ij}\,\vect V_j,
\qquad
\alpha_{ij}=\mathrm{softmax}\!\left(\frac{\vect q_i^{\!\top}\vect k_j}{\sqrt{d_k}}\right).
\end{equation}
The learned projection matrices \(W_Q,W_K,W_V,W_O\) and the MLP weights are fixed for a given model during inference; the input-dependent variables \(\alpha_{ij}\) determine which value states are read and mixed. Thus, standard attention provides a powerful form of \emph{adaptive state routing}. Unlike VN, these adaptive coefficients are generated feedforward rather than inferred by energy minimization, and they are not subject to an explicit sparsity pressure or competitive selection among reusable operator atoms.

\paragraph{Adaptive operator composition in VN.}
In VN, the input-dependent variables are sparse weight-atom coefficients. After the fast settle, the coefficients \(g_l^\star(x_l)\) select rank-1 atoms and compose a sample-specific operator,
\begin{equation}
\mat W_l^\star(x_l)
=\sum_k g_{l,k}^\star(x_l)\,u_{l,k}s_{l,k}^{\top}.
\end{equation}
The adaptive variable therefore specifies \emph{which reusable operators are composed by the inferred code}, not only which states are read. The same inferred coefficients later gate the slow local updates on those active atoms, so computation and plasticity are organized around the same atom-level responsibility pattern.

\paragraph{Low rank plays different roles.}
\emph{VN.} Each weight atom is rank-1 in feature space, so the induced operator lies in
\begin{equation}
\mat W_l^\star(x_l)\in\mathrm{span}\{u_{l,k}s_{l,k}^{\top}\}_{k=1}^{K_l}.
\end{equation}
Low rank constrains the operator atoms exposed after inference; sparse gains \(g_{l,k}^\star(x_l)\) compose these atoms into a sample-specific operator description.

\noindent
\emph{Transformer.} For one attention head, pre-softmax scores form a bilinear token-token matrix,
\begin{equation}
\mat S=\frac{\mat Q\mat K^{\!\top}}{\sqrt{d_k}}
=\frac{1}{\sqrt{d_k}}\sum_{m=1}^{d_k}\mat Q_{(:,m)}\,\mat K_{(:,m)}^{\!\top}.
\end{equation}
This low-rank structure shapes routing capacity over tokens. Multi-head attention combines several such routing patterns, but the learned projection matrices themselves are not selected as a sparse superposition of reusable weight atoms.

\paragraph{Comparison summary.}
\begin{center}
\begin{tabular}{p{0.24\linewidth} p{0.34\linewidth} p{0.34\linewidth}}
\hline
\textbf{Aspect} & \textbf{VN (operators inferred from sparse atoms)} & \textbf{Transformer (operators generated by attention)} \\
\hline
Adaptive variable &
Sparse coefficients \(g_l^\star(x_l)\) &
Attention weights \(\alpha_{ij}\) \\
What it acts on &
Reusable rank-1 operator atoms \(u_{l,k}s_{l,k}^{\top}\) &
Token/value states under fixed feature projections \\
How it is obtained &
Iterative sparse inference under an explicit energy and sparsity penalty &
Feedforward softmax attention \\
Does it gate learning? &
Yes: the same coefficients later gate slow local updates on active atoms &
No: attention weights route computation only \\
\hline
\end{tabular}
\end{center}

\paragraph{Competition during inference.}
Transformer attention weights are generated feedforward and normalized by softmax, so they redistribute mass over available states but do not arise from an explicit sparse energy minimization over reusable operator atoms. In VN, by contrast, the coefficients \(g_l^\star\) are obtained by minimizing a sparsity-regularized energy, so weight atoms compete directly for responsibility under an explicit code budget.

\paragraph{Conceptual summary.}
The core distinction is that transformers generate input-dependent state-mixing operators feedforward, whereas VN infers a sparse, competitive composition of reusable rank-1 operator atoms. In VN, the converged coefficients \(g_l^\star(x_l)\) select one sample-specific operator from the span
\begin{equation}
\mathcal W_l=\Big\{\sum_k g_k\,u_{l,k}s_{l,k}^{\top}\;\Big|\;g\in\R^{K_l}\Big\}.
\end{equation}
The same coefficients then gate the corresponding slow local updates. Transformer attention, by contrast, adapts computation by redistributing activity over states under fixed feature-space projections, without a matching responsibility-gated operator update. Low-rank structure therefore plays different roles: it governs operator composition and responsibility-gated plasticity in VN, and state routing in transformers.

\section{Fast Inference of Weight Atom Coefficients in Deep Vector Networks}
\label{app:fast_inference}
This section explains how a deep VN settles through hierarchical recurrent sparse inference. The key point is that layers are not fully optimized one at a time; instead, sparse coefficient updates are interleaved across the hierarchy through repeated upward and downward sweeps until the full stack reaches a consistent fixed point.

For layer \(l\), the energy is
\begin{equation}
E_l(g_l\mid x_l,h_{l,\mathrm{target}})
=\tfrac12\|x_l-S_l g_l\|_2^2+\lambda_l\|g_l\|_1+\tfrac12\|h_{l,\mathrm{target}}-U_l g_l\|_2^2.
\end{equation}
The smooth part \(f_l(z)=\tfrac12\|x_l-S_l z\|^2+\tfrac12\|h_{l,\mathrm{target}}-U_l z\|^2\) has gradient
\begin{equation}
\nabla f_l(z)=S_l^\top(S_l z-x_l)+U_l^\top(U_l z-h_{l,\mathrm{target}}),
\end{equation}
where the \(U\) term is omitted when \(h_{l,\mathrm{target}}\) is not defined (top layer). A proximal-gradient update is then
\begin{equation}
z \leftarrow \mathrm{prox}_{\eta\lambda_l\|\cdot\|_1}\!\left(z-\eta\,\nabla f_l(z)\right)
=\mathrm{soft}\!\left(z-\eta\,\nabla f_l(z),\ \eta\lambda_l\right).
\end{equation}
Choosing \(\eta \le 1/L\) with \(L=\|\;S_l^\top S_l+U_l^\top U_l\;\|_2\) yields a descent step for the coupled update. In practice we use \(\eta_l \approx 1/L\); during the first sweep, before top-down targets have been refreshed, we optionally use a bottom-up-only step size \(\eta_{s,l}\approx 1/\|S_l^\top S_l\|_2\). When a fixed code budget is required, the proximal update may be followed by a hard Top-\(k\) projection, as in Eq.~\eqref{eq:vn_update_hard}. Thus each layer update is an ISTA/FISTA-style sparse inference step on \(g_l\), not a full layerwise convergence before the next layer is touched.

The settle therefore proceeds as a distributed hierarchical recurrent process. Starting at the bottom layer, one sparse update step is taken for \(g_1\); the resulting dense upward message \(h_1=\phi_1(U_1g_1)\), identified with the next-layer input \(x_2\), is then propagated upward so that layer \(2\) can take one sparse update step, and so on until the top layer is reached. The process then reverses direction: top-down targets are refreshed, \(h_{l,\mathrm{target}}\leftarrow S_{l+1}g_{l+1}\), and the same kind of sparse update is applied during a downward sweep. One complete upward and downward traversal constitutes a single inference sweep, and the network settles through many such recurrent sweeps rather than by solving each layer in isolation. Operationally, this resembles a recurrent encoder--decoder settle: activity first propagates upward toward progressively sparser codes, then reconstructed top-down targets are sent back down to refine the lower-layer coefficients. The analogy is only heuristic, however, because both passes are part of one coupled sparse inference dynamics rather than separate feedforward encoder and decoder networks.

Different schedules for this recurrent inference are possible. The simplest schedule, and the one closest to the implementation described in the main text, updates each layer once per upward pass and once per downward pass. More asynchronous variants are also possible: once a lower layer has been initialized, it can continue receiving recurrent ISTA/FISTA updates while higher layers are still being initialized, so that by the time the top layer is reached, lower layers may already have undergone several refinement steps. More generally, after an initial feedforward initialization, one can transition into a fully recurrent regime in which all layers are updated continuously at every inference time step. In all cases, the best way to understand the VN settle is as a coupled hierarchical recurrent inference dynamics in which sparse operator coefficients across all layers gradually co-adapt until the hierarchy reaches a consistent fixed point.

Other first-order variants (including accelerated proximal-gradient updates) are equally applicable as heuristic solvers with energy tracking and optional accept/reject safeguards \cite{beck2009fast}. The \(L_1\) term provides the main sparsity pressure that encourages weight-atom selection. In our experiments, hidden-layer sparsity is controlled by the soft-thresholding update itself, while any strict hard support cap is reserved for the top layer. We usually initialize \(g_l=0\) and \(h_{l,\mathrm{target}}=0\), but the implementation also supports warm starts, resumed code states, and optional nonzero random initialization.

\section{Slow Local Dictionary Updates: The Atomic-Hebb Rule}
\label{app:atomic_hebb}
This section derives the Atomic-Hebb local dictionary update, i.e.\ the residual\(\times\)atom-coefficient rule used after the fast settle. Here \(g_l^\star\) denotes the converged sparse code after fast inference has settled. The key point is that this plasticity is weight-atom local: each reusable operator atom is updated only through its own inferred coefficient and the residual available at that layer.
With \(g_l^\star\) fixed after the fast settle, the reconstruction term has gradient
\begin{equation}
\nabla_{S_l}\tfrac12\|x_l-S_l g_l^\star\|_2^2=-(x_l-S_l g_l^\star)\,(g_l^\star)^\top=-\,r_l^x (g_l^\star)^\top,
\end{equation}
and the top-down term has gradient (when defined)
\begin{equation}
\nabla_{U_l}\tfrac12\|h_{l,\mathrm{target}}-U_l g_l^\star\|_2^2=-(h_{l,\mathrm{target}}-U_l g_l^\star)\,(g_l^\star)^\top=-\,r_l^h (g_l^\star)^\top.
\end{equation}
Taking a direct gradient step on \((S_l,U_l)\) yields the local residual\(\times\)atom-coefficient updates used in Sec.~\ref{sec:slow}. The subtle but important point is that Atomic-Hebb is not fully independent per atom in the strict sense, because the residuals
\begin{equation}
r_l^x = x_l-\sum_j g^\star_{l,j}s_{l,j},
\qquad
r_l^h = h_{l,\mathrm{target}}-\sum_j g^\star_{l,j}u_{l,j}
\end{equation}
depend on the full sparse reconstruction and therefore on all active atoms. The update is therefore not
\begin{equation}
\Delta A^{(l,i)}\sim \text{local information of atom } i
\end{equation}
in isolation. The key point instead is that credit assignment is localized at the level of \emph{operator responsibility}. The coefficient \(g^\star_{l,i}\) determines which atom participated and how strongly it contributed, while the residual measures how much unexplained error remains after the joint sparse explanation. Column-wise, the rule can therefore be read as
\begin{equation}
\Delta s_{l,i}\propto g^\star_{l,i}\,r_l^x,
\qquad
\Delta u_{l,i}\propto g^\star_{l,i}\,r_l^h.
\end{equation}
Every active atom sees the same layer-local residual field, but scaled by its own sparse responsibility.

This is fundamentally different from backpropagation or predictive coding, where the update of a weight depends on correlations between activities of different layers, schematically
\begin{equation}
\Delta W_l \sim \delta_{l+1}x_l^\top,
\end{equation}
where \(\delta_{l+1}\) denotes the output error or upstream error signal carried by the next layer's activities. In VN, the update acts only on atoms of one layer, and each atom is weighted by its own inferred coefficient \(g^\star_{l,i}\). The rule is therefore not ``fully atom-independent,'' because the shared residual still depends on the joint sparse explanation, but credit assignment is localized at the level of operator responsibility. This is closely analogous to sparse coding itself: atoms compete to explain the input, the residual reflects the remaining unexplained signal, and only active atoms receive credit.

The important distinction is therefore not that each atom learns in total isolation, but that the interaction is within-layer, reconstruction-based, sparse, and competitive, rather than cross-layer, state-sharing, or dependent on gradient transport through multiple transformations. That is why the rule is still meaningfully layer-local and weight-atom-local even though the residual is shared. In the implementation, the same residual\(\times\)atom-coefficient terms are used as local gradients for Adam, followed by per-column renormalization (and optional DC removal in pixel space).

\paragraph{Renormalization and DC removal.}
After each update, columns of \(S_l\) and \(U_l\) are renormalized to unit norm. For pixel-space dictionaries (e.g. \(16{\times}16\) patches) we optionally remove the DC component (the mean / zero-frequency content) of each \(S_l\) column before renormalization. These procedural constraints play the role of gain/variance control and keep sparsity statistics well-defined.

\section{Theoretical Foundations and Proofs}
\label{app:theory}

This section collects the main formal statements underlying the Vector Network architecture. It covers compositional reconstruction, the Atomic-Hebbian update, convergence of fast inference, capacity bounds, and the role of top-down coupling.

We provide formal derivations and discussions for five core elements of the Vector Network architecture: (1)~why linear sparse reconstruction algebraically guarantees compositional generalization under superposition, (2)~why the local Atomic-Hebbian updates are exact gradients of the equilibrium energy, (3)~convergence guarantees of the fast inference dynamics, (4)~the effective capacity bounds induced by sparse weight-atom inference, and (5)~the top-down coupling viewed as predictive coding. Where relevant we also note gaps between the idealized theory and the practical implementation.

\subsection{Compositional Generalization via Linear Sparse Reconstruction}
\label{app:theory:composition}

The core empirical result motivating this section is that VNs generalize zero-shot to novel combinations of mechanisms, such as multi-bump superpositions or mixed-force OOD cases, whereas dense and factorized models fail. We formalize this in two parts: first, the algebraic preservation of superposition through the linear decoder; second, the conditions under which sparse inference recovers the correct composition.


\paragraph{Part 1: Superposition preservation.}
The VN reconstructs its output as a linear combination of dictionary atoms:
$\hat{\mathbf{y}} = \sum_{i=1}^K g_i \mathbf{s}_i = \mat{S}\mathbf{g}$.
This reconstruction is \emph{linear} in the code $\mathbf{g}$, which immediately yields:

\begin{proposition}[Superposition Preservation]
\label{prop:superposition}
Let $\mathbf{g}_1, \mathbf{g}_2 \in \R^K$ be sparse codes that produce outputs $\hat{\mathbf{y}}_1 = \mat{S}\mathbf{g}_1$ and $\hat{\mathbf{y}}_2 = \mat{S}\mathbf{g}_2$ respectively. Then the code $\mathbf{g}_1 + \mathbf{g}_2$ produces the exact superposition:
\begin{equation}
    \mat{S}(\mathbf{g}_1 + \mathbf{g}_2) = \mat{S}\mathbf{g}_1 + \mat{S}\mathbf{g}_2 = \hat{\mathbf{y}}_1 + \hat{\mathbf{y}}_2
\end{equation}
\end{proposition}

\begin{proof}
Immediate from linearity of matrix multiplication: $\mat{S}(\mathbf{g}_1 + \mathbf{g}_2) = \mat{S}\mathbf{g}_1 + \mat{S}\mathbf{g}_2$.
\end{proof}

The key insight is that any nonlinear decoder, including standard CNN decoders and factorized outer-product architectures, breaks this identity. For a nonlinear decoder $D$, the output $D(\mathbf{g}_1 + \mathbf{g}_2) \neq D(\mathbf{g}_1) + D(\mathbf{g}_2)$ in general, and compositional generalization becomes architecture-dependent rather than guaranteed.

For the outer-product factorized decoder specifically, the rank-1 output $f(\mathbf{x})\,g(\mathbf{y})$ applied to a superimposed input $(\mathbf{x}_1{+}\mathbf{x}_2, \mathbf{y}_1{+}\mathbf{y}_2)$ necessarily produces $N^2$ cross-terms instead of $N$ correct terms, as demonstrated in Section~\ref{sec:blob}.


\paragraph{Part 2: Sparse recovery of novel compositions.}
Proposition~\ref{prop:superposition} guarantees that \emph{if} the correct composite code $\mathbf{g}_1 + \mathbf{g}_2$ is found, the output is exact. The remaining question is whether inference can find this code when the composition was never seen during training.

\textbf{Setup.} Let the true generative process be $\mathbf{y} = \sum_{j=1}^C c_j \mat{W}^{(j)}_{\text{true}} \mathbf{x}$ where $\mathbf{c} \in \{0,1\}^C$ indicates active mechanisms. Assume the dictionary contains atoms $\mat{A}^{(i)} \approx \mat{W}^{(j)}_{\text{true}}$ for each mechanism $j$. Define the dynamic dictionary $\mat{\Phi}(\mathbf{x}) \in \R^{m \times K}$ with columns $\mat{\Phi}_i(\mathbf{x}) = \mat{A}^{(i)}\mathbf{x}$. Theorem~\ref{thm:composition} below analyses recovery for this idealized output-side objective; it provides a guide to when sparse recovery succeeds rather than a literal analysis of VN's implemented layer-local energy ($\frac{1}{2}\|x - Sg\|^2 + \frac{1}{2}\|h_\text{target} - Ug\|^2 + \lambda\|g\|_1$, Eq.~\eqref{eq:layer_energy}).

\begin{theorem}[Exact Compositional Recovery]
\label{thm:composition}
Let $\mu(\mat{\Phi}(\mathbf{x}))$ denote the mutual coherence: $\mu = \max_{i \neq j} \frac{|\mat{\Phi}_i^\top \mat{\Phi}_j|}{\|\mat{\Phi}_i\|\|\mat{\Phi}_j\|}$. Assume (a) the dictionary atoms approximate the true mechanisms with bounded residual error $\mathbf{e}$, and (b) $\lambda$ is chosen with $\lambda \gtrsim \|\mathbf{\Phi}^\top \mathbf{e}\|_\infty$. If the number of active OOD mechanisms $k_{\text{OOD}} = \|\mathbf{c}_{\text{OOD}}\|_0$ satisfies
\begin{equation}
    k_{\text{OOD}} < \frac{1}{2} \left( 1 + \frac{1}{\mu(\mat{\Phi}(\mathbf{x}))} \right),
\end{equation}
then the $L_1$-regularized inference
$\mathbf{g}^\star = \arg\min_{\mathbf{g}} \frac{1}{2}\|\mathbf{y} - \mat{\Phi}(\mathbf{x})\mathbf{g}\|_2^2 + \lambda\|\mathbf{g}\|_1$
recovers the true sparse support of $\mathbf{c}_{\text{OOD}}$.
\end{theorem}

Conditions under which assumptions (a) and (b) hold in practice are discussed in Remark~\ref{remark:conditions}.

\begin{proof}
The coherence bound is the Exact Recovery Condition for Basis Pursuit~\cite{donoho2003optimally,tropp2004greed}. In the noiseless equality-constrained case the $L_1$ relaxation matches the $L_0$ combinatorial minimum exactly; in the noisy LASSO regime, support recovery further requires the regularizer to dominate the residual, $\lambda \gtrsim \|\mathbf{\Phi}^\top \mathbf{e}\|_\infty$, which is the role of assumption (b).
\end{proof}

\begin{remark}[Conditions and limitations]
\label{remark:conditions}
Theorem~\ref{thm:composition} requires two conditions that are not automatically satisfied:
\begin{enumerate}[label=(\roman*)]
    \item \textbf{Dictionary quality:} The atoms must faithfully represent the true mechanisms. In practice, this depends on the Atomic-Hebbian learning dynamics (Section~\ref{app:theory:gradients}) and sufficient training data covering each mechanism individually.
    \item \textbf{Coherence:} For the rank-1 atoms $A^{(i)} = u_i s_i^\top$, the columns of the dynamic dictionary factorize as $\mat{\Phi}_i(\mathbf{x}) = (s_i^\top\mathbf{x})\,u_i$, so the scalar projections cancel during normalization and $\mu(\mat{\Phi}(\mathbf{x})) = \mu(U)$ is \emph{independent of the test input} (modulo signs of the projections). The coherence condition therefore reduces to a static incoherence requirement on the output dictionary $U$, which holds generically when atoms encode distinct mechanisms with different spatial/spectral signatures.
\end{enumerate}
The practical implication is that compositional generalization is an \emph{emergent} consequence of successful dictionary learning and sufficient mechanism diversity, not an unconditional architectural guarantee.
\end{remark}

\subsection{Atomic-Hebbian Updates as Exact Implicit Gradients}
\label{app:theory:gradients}

The VN updates its dictionary $\mat{S}$ using a local rule $\Delta \mat{S} \propto \mathbf{r}^x (\mathbf{g}^\star)^\top$ (residual $\times$ activations) without backpropagating through the iterative inference solver. The result that such residual-times-code rules are the exact implicit gradient of an equilibrium energy is a long-established envelope-theorem argument in sparse coding~\cite{olshausen1996emergence}, online dictionary learning~\cite{mairal2009online}, and deep equilibrium models~\cite{bai2019deep}; we restate it here for completeness because it is what underlies the locality of the Atomic-Hebbian update and the validity of treating fast inference as a black box during slow learning.

\textbf{Setup.} Let the layer energy be $E(\mathbf{g}, \mat{S}) = f(\mathbf{g}, \mat{S}) + \lambda\|\mathbf{g}\|_1$ with smooth component $f(\mathbf{g}, \mat{S}) = \frac{1}{2}\|\mathbf{x} - \mat{S}\mathbf{g}\|_2^2$. Inference finds $\mathbf{g}^\star(\mat{S}) = \arg\min_{\mathbf{g}} E(\mathbf{g}, \mat{S})$, and learning minimizes $\mathcal{L}(\mat{S}) = E(\mathbf{g}^\star(\mat{S}), \mat{S})$.

\begin{theorem}[Exact Implicit Gradients via Danskin's Theorem]
\label{thm:danskin}
Assume: (i)~the columns of $\mat{S}$ are in general position (no $d{+}1$ columns lie on a common hyperplane), ensuring the LASSO minimizer $\mathbf{g}^\star(\mat{S})$ is unique \cite{tibshirani2013lasso}; (ii)~the smooth component $f(\mathbf{g}, \mat{S})$ is jointly continuous in $(\mathbf{g}, \mat{S})$. Then:
\begin{equation}
    \frac{d \mathcal{L}}{d \mat{S}} = \left. \frac{\partial f(\mathbf{g}, \mat{S})}{\partial \mat{S}} \right|_{\mathbf{g}=\mathbf{g}^\star} = -\mathbf{r}^x (\mathbf{g}^\star)^\top
\end{equation}
where $\mathbf{r}^x = \mathbf{x} - \mat{S}\mathbf{g}^\star$ is the reconstruction residual.
\end{theorem}

\begin{proof}
The objective $E(\mathbf{g}, \mat{S})$ is non-smooth in $\mathbf{g}$ because of the $L_1$ term, so the classical chain rule does not apply directly. Instead, the optimality of $\mathbf{g}^\star(\mat{S})$ gives the subgradient condition $0 \in \partial_{\mathbf{g}} E(\mathbf{g}^\star, \mat{S})$, where $\partial_{\mathbf{g}}$ denotes the subdifferential. Under assumptions (i)--(ii), the generalized Danskin's theorem~\cite{danskin1967theory,bertsekas1999nonlinear} applies: the total derivative of $\mathcal{L}(\mat{S}) = E(\mathbf{g}^\star(\mat{S}), \mat{S})$ reduces to the partial derivative of the smooth component $f$ at the optimum,
\begin{equation}
    \frac{d \mathcal{L}}{d \mat{S}} = \left. \frac{\partial f}{\partial \mat{S}} \right|_{\mathbf{g}=\mathbf{g}^\star} = -(\mathbf{x} - \mat{S}\mathbf{g}^\star)(\mathbf{g}^\star)^\top = -\mathbf{r}^x (\mathbf{g}^\star)^\top.
\end{equation}
The Atomic-Hebbian update $\Delta \mat{S} \propto \mathbf{r}^x (\mathbf{g}^\star)^\top$ is therefore exact gradient descent on the equilibrium energy, not a heuristic approximation.
\end{proof}

\begin{remark}[Extension to U]
The same argument applies to the lateral dictionary $\mat{U}$ via the upper-consistency energy $\frac{1}{2}\|\mathbf{h}_{\text{target}} - \mat{U}\mathbf{g}\|^2$, yielding $\Delta \mat{U} \propto \mathbf{r}^h (\mathbf{g}^\star)^\top$. For the full VN energy combining both terms, the Atomic-Hebbian updates on $\mat{S}$ and $\mat{U}$ correspond to block coordinate descent on the equilibrium landscape.
\end{remark}

\begin{remark}[Biological plausibility]
The update depends only on quantities locally available at the neuron: the reconstruction residual $\mathbf{r}^x$ (error signal) and the converged activation $\mathbf{g}^\star$ (post-synaptic activity). No backward pass through the inference dynamics is required, making Atomic-Hebb compatible with the locality usually associated with Hebbian plasticity \cite{oja1982simplified}.
\end{remark}

\begin{remark}[Why Atomic-Hebbian updates generalize while backprop memorizes]
\label{remark:hebbian_vs_bp}
Theorem~\ref{thm:danskin} proves that the Atomic-Hebbian update is the \emph{correct} gradient of the equilibrium energy, but does not explain why the resulting dictionary generalizes OOD while a backprop-trained linear decoder with identical architecture memorizes (Section~\ref{sec:blob}). We offer the following intuition. The Atomic-Hebbian rule $\Delta \mat{S} \propto \mathbf{r}^x (\mathbf{g}^\star)^\top$ correlates residuals with activities that are themselves computed \emph{from the same dictionary} via sparse inference. This self-consistency couples the learning signal to the dictionary's current representational structure: atoms are updated to reduce error in directions that the dictionary already spans. \cite{baldi1989neural} showed that for linear networks, such local update rules converge to global optima of the reconstruction objective without the saddle-point issues of generic gradient descent, providing a partial theoretical basis for why the dictionary is well-conditioned. More recently, \cite{whittington2017approximation} proved that inference-then-Hebbian-update in predictive coding networks can approximate backpropagation under specific structural conditions, suggesting the two learning rules share fixed points but may differ in their convergence basins and implicit regularization. In the VN, the two-phase structure, infer first, then update only the selected atoms, imposes an implicit curriculum: the dictionary can only learn what the sparse inference can already partially represent, preventing the global task-loss memorization that plagues end-to-end backpropagation. A complete characterization of this implicit regularization, and why it favors compositionally generalizable dictionaries, remains an open theoretical question.
\end{remark}

\subsection{Convergence of Fast Inference}
\label{app:theory:convergence}

\textbf{Setup.} The layer-local energy is:
\begin{equation}
    E_l(\mathbf{g}_l) = \underbrace{\frac{1}{2}\|\mathbf{x}_l - \mat{S}_l \mathbf{g}_l\|_2^2 + \frac{1}{2}\|\mathbf{h}_{l,\text{target}} - \mat{U}_l \mathbf{g}_l\|_2^2}_{f_l(\mathbf{g}_l)} + \lambda_l\|\mathbf{g}_l\|_1
\end{equation}

\begin{theorem}[Convergence of Proximal Gradient Inference]
\label{thm:convergence}
The ISTA update $\mathbf{g}_l^{(t+1)} = \mathrm{prox}_{\eta_l\lambda_l\|\cdot\|_1}\!\bigl(\mathbf{g}_l^{(t)} - \eta_l \nabla f_l(\mathbf{g}_l^{(t)})\bigr)$ with step size $\eta_l \le L^{-1}$, where $L = \|\mat{S}_l^\top \mat{S}_l + \mat{U}_l^\top \mat{U}_l\|_2$ is the Lipschitz constant of $\nabla f_l$, satisfies:
\begin{equation}
    E_l(\mathbf{g}_l^{(t+1)}) \le E_l(\mathbf{g}_l^{(t)}) - \frac{1}{2\eta_l} \|\mathbf{g}_l^{(t+1)} - \mathbf{g}_l^{(t)}\|_2^2,
\end{equation}
so the energy decreases monotonically and converges to the global minimum at rate $O(1/t)$ in the objective error. Using FISTA acceleration~\cite{beck2009fast}, the objective-error rate improves to $O(1/t^2)$, although the accelerated variant does not preserve per-iteration monotone descent.
\end{theorem}

\begin{proof}
The smooth component $f_l$ is convex quadratic with Hessian $\nabla^2 f_l = \mat{S}_l^\top\mat{S}_l + \mat{U}_l^\top\mat{U}_l$. Its gradient is $L$-Lipschitz with $L = \lambda_{\max}(\nabla^2 f_l)$. The descent lemma for proximal gradient methods \cite{beck2009fast} gives the stated inequality for $\eta_l \le 1/L$. Since $E_l$ is bounded below and monotonically non-increasing, the iterates converge to a fixed point. Convexity of $E_l$ ensures this fixed point is the global minimum. Convergence rates follow from \cite[Theorems 3.1 and 4.4]{beck2009fast}.
\end{proof}

\begin{remark}[Soft-$k$ vs.\ hard-$k$ regimes]
\label{remark:topk}
Theorem~\ref{thm:convergence} covers the soft-$k$ regime (pure proximal $\ell_1$, Eq.~\eqref{eq:vn_update_soft}), used in our experiments at hidden layers. The hard-$k$ regime of Eq.~\eqref{eq:vn_update_hard}, applied at the top layer when a strict cardinality budget is required, composes the proximal step with a top-$k$ projection. This introduces a non-convex constraint $\|\mathbf{g}\|_0 \le k$, so convergence to the \emph{global} optimum is no longer guaranteed; the solver may converge to a stationary point of the combinatorial problem. We note three points:
\begin{enumerate}[label=(\roman*)]
    \item The energy still decreases monotonically at each step (the top-$k$ projection is a valid proximal operator for the indicator function of the $\ell_0$ ball), ensuring convergence to a stationary point.
    \item For sufficiently incoherent dictionaries, the soft-$k$ and hard-$k$ regimes select the same support \cite{tropp2004greed}, so the two regimes are equivalent under well-conditioned dictionaries.
    \item Our multi-layer settling interleaves updates across layers (block coordinate descent), which inherits the monotone descent property layer-by-layer but lacks a global convergence rate guarantee for the joint objective. The closest theoretical reference is the multi-layer convolutional sparse coding framework of \cite{sulam2019multilayer}, which provides convergence guarantees for a related hierarchical pursuit under incoherence conditions; extending those results to the VN's bidirectional ($\mat{S}/\mat{U}$) atom structure is an open problem.
\end{enumerate}
\end{remark}

\subsection{Effective Capacity Bound}
\label{app:theory:capacity}

Standard dense layers deploy their entire parameter budget for every input. We show that the VN's sparse inference structurally restricts the per-sample hypothesis space.

\begin{proposition}[Rank Bound for the Synthesized Operator]
\label{prop:rank}
For any input $\mathbf{x}$, let $\mat{W}^\star(\mathbf{x}) = \mat{U}\,\mathrm{diag}(\mathbf{g}^\star)\,\mat{S}^\top$ denote the operator synthesized from the inferred sparse code. Its rank is bounded by the activation sparsity:
\begin{equation}
    \mathrm{rank}(\mat{W}^\star(\mathbf{x})) \le \|\mathbf{g}^\star\|_0 \le k
\end{equation}
where $k$ is the top-$k$ sparsity constraint. The corresponding per-sample parameter budget is bounded by $k(d + m)$, compared to $md$ for a dense layer.
\end{proposition}

\begin{proof}
$\mat{W}^\star = \mat{U}\,\mathrm{diag}(\mathbf{g}^\star)\,\mat{S}^\top$. By sub-multiplicativity: $\mathrm{rank}(\mat{W}^\star) \le \mathrm{rank}(\mathrm{diag}(\mathbf{g}^\star)) = \|\mathbf{g}^\star\|_0 \le k$.
\end{proof}

This rank bound characterizes the sample-specific operator induced by the inferred sparse code, not a literal forward multiplication $\mat{W}^\star(\mathbf{x})\mathbf{x}$.

\begin{remark}[Capacity and reuse]
The condition $k \ll \frac{md}{m+d}$ (where $d$ is the input dimension and $m$ the output dimension) ensures that the per-sample transformation is strictly lower-rank than what a dense layer could represent. Combined with a shared dictionary learned across all training samples, this forces the slow learning phase to discover reusable, compositionally generalizable atoms rather than sample-specific memorization patterns. From a statistical learning perspective, the Rademacher complexity of $k$-sparse linear models scales as $\mathcal{O}(\sqrt{k \log(K/k) / n})$ rather than $\mathcal{O}(\sqrt{d / n})$ for dense models \cite{bartlett2002rademacher}, providing a formal basis for the improved generalization of sparse representations when $k \ll d$.
\end{remark}

\begin{remark}[Failure conditions]
\label{remark:failure}
The rank constraint is also a source of failure when violated:
\begin{enumerate}[label=(\roman*)]
    \item \textbf{Insufficient sparsity:} When the target function requires more than $k$ independent components (e.g., a scene with $k{+}1$ superimposed objects), the VN necessarily under-fits because its effective weight matrix cannot span the required subspace.
    \item \textbf{Dictionary collapse:} When atoms become near-duplicate during Atomic-Hebbian learning, the mutual coherence $\mu(\mat{S})$ (defined in Theorem~\ref{thm:composition}) increases and the recovery condition may be violated, causing inference to select incorrect atoms.
    \item \textbf{Incomplete convergence:} The practical settle uses a finite number of sweeps. If inference terminates before the energy has converged sufficiently, the code $\mathbf{g}$ may not satisfy the optimality conditions required by Theorems~\ref{thm:danskin} and~\ref{thm:composition}.
\end{enumerate}
In practice, the decorrelation step in the slow update and the energy-tracking accept/reject mechanism (Section~\ref{sec:fast}) mitigate conditions (ii) and (iii), but do not eliminate them in the worst case.
\end{remark}

\subsection{Top-Down Coupling}
\label{app:theory:topdown}

The VN's top-down information enters the lower-layer code $\mathbf{g}_l$ through the consistency term $\frac{1}{2}\|\mathbf{h}_{l,\mathrm{target}} - \mat{U}_l \mathbf{g}_l\|_2^2$ in the layer energy $E_l$. During the downward sweep (Sec.~\ref{sec:fast}), the target is refreshed from the higher-layer reconstruction, $\mathbf{h}_{l,\mathrm{target}} \leftarrow \mat{S}_{l+1}\mathbf{g}_{l+1}$, and a coupled proximal step on $\mathbf{g}_l$ uses both bottom-up and top-down residuals.

\begin{remark}[Top-down coupling via energy consistency]
\label{remark:topdown}
The top-down term contributes the gradient component $\mat{U}_l^\top(\mat{U}_l\mathbf{g}_l - \mathbf{h}_{l,\mathrm{target}})$ to the smooth part of $E_l$, so the proximal updates of Eqs.~\eqref{eq:vn_update_soft}--\eqref{eq:vn_update_hard} automatically incorporate a top-down correction whenever $\mathbf{h}_{l,\mathrm{target}}$ has been refreshed. This integrates predictive-coding-style top-down information directly into the energy descent rather than as a separate update rule.

Unlike classical predictive coding \cite{rao1999predictive}, the VN does not enforce strict weight transport between forward and backward pathways. The lateral dictionary $\mat{U}_l$ used in the top-down consistency term is learned independently from any forward Jacobian; implementations may also use untied dictionaries for the upward message and the downward consistency, in which case the top-down dictionary plays a role analogous to a learned feedback weight as in the Feedback Alignment principle \cite{lillicrap2016random}. We do not prove convergence of any such alignment for the VN's specific dynamics; an empirical ablation of the top-down coupling strength is reported in Appendix~\ref{app:td_ablation}. Conceptually, the VN's fast--slow separation mirrors Deep Equilibrium Models \cite{bai2019deep}, where forward inference computes a fixed point and learning differentiates through the equilibrium; the key difference is that the VN uses an explicit energy function with convergence guarantees (Theorem~\ref{thm:convergence}) rather than implicit differentiation through a black-box fixed-point solver.
\end{remark}

\begin{remark}[Connections to related frameworks]
The VN's inference-learning separation connects to several established theoretical frameworks:
\begin{itemize}
    \item \textbf{Equilibrium models:} The fast-slow separation mirrors Deep Equilibrium Models \cite{bai2019deep}, but with an explicit energy function and convergence guarantee rather than implicit differentiation.
    \item \textbf{Sparse coding:} The single-layer inference is equivalent to the LASSO / Basis Pursuit \cite{chen1998atomic}, with the multi-layer extension following the hierarchical sparse coding framework of \cite{sulam2019multilayer}.
    \item \textbf{Energy-based models:} The overall architecture fits within the energy-based learning framework of \cite{lecun2006tutorial}, where inference minimizes energy and learning minimizes the minimum energy.
    \item \textbf{Predictive coding:} The top-down coupling (Remark~\ref{remark:topdown}) connects to the predictive coding framework \cite{rao1999predictive,whittington2017approximation}, where higher layers predict lower-layer activity and residuals drive learning.
\end{itemize}
\end{remark}

\section{Spatial Decoding: Extended Results}
\label{app:fiete_extended}

This section extends the spatial decoding benchmark from Fig.~\ref{fig:blob_fiat} with additional ablations on rank, depth, encoding, and holdout structure. The aim is to show more explicitly which parts of the VN design are responsible for the off-support robustness in the bump task.


\paragraph{Increasing architectural rank does not help.}
A rank-$N$ factorized model trained on rank-1 data collapses to rank-1:
extra streams become redundant copies
(Table~\ref{tab:fiete_rankn}).

\begin{table}[h]
\centering
\caption{\textbf{Rank ablation.} All trained on single bumps.}
\label{tab:fiete_rankn}
\smallskip
\small
\begin{tabular}{@{}lccc@{}}
\toprule
Architecture & $N{=}1$ \ood{} & $N{=}2$ & $N{=}3$ \\
\midrule
Rank-1 & $2.0{\times}10^{-5}$ & $7.3{\times}10^{-3}$ & $1.1{\times}10^{-2}$ \\
Rank-2 & $3.0{\times}10^{-5}$ & $7.8{\times}10^{-3}$ & $1.1{\times}10^{-2}$ \\
Rank-3 & $2.6{\times}10^{-5}$ & $8.1{\times}10^{-3}$ & $1.2{\times}10^{-2}$ \\
Full 2D VN & $4.0{\times}10^{-3}$ & $\mathbf{3.7{\times}10^{-4}}$ & $\mathbf{4.4{\times}10^{-4}}$ \\
\bottomrule
\end{tabular}
\end{table}


\paragraph{VN architecture ablation.}
\label{app:fiete_vn_ablations}
Table~\ref{tab:fiete_vn_depth}: depth~2 captures most benefit over
depth~1; $k{=}8$ outperforms $k{=}16$ (sparser codes compose better);
mix${=}0.7$ improves VN-3L by $23\%$ because the deeper hierarchy
benefits from stronger top-down refinement of reconstruction codes.

\begin{table}[h]
\centering
\caption{\textbf{VN depth, sparsity, and top-down ablation.}}
\label{tab:fiete_vn_depth}
\smallskip
\small
\begin{tabular}{@{}lccccc@{}}
\toprule
Config & Depth & $k$ & Ratio & $N{=}2$ & $N{=}3$ \\
\midrule
VN-1L & 1 & 16 & $1.16{\times}$ & $2.65{\times}10^{-4}$ & $3.01{\times}10^{-4}$ \\
VN-2L & 2 & 8 & $1.08{\times}$ & $9.53{\times}10^{-5}$ & $9.88{\times}10^{-5}$ \\
VN-3L & 3 & 8 & $1.07{\times}$ & $9.93{\times}10^{-5}$ & $1.01{\times}10^{-4}$ \\
VN-3L (mix${}=0.7$) & 3 & 8 & $\mathbf{1.06{\times}}$ & $\mathbf{7.62{\times}10^{-5}}$ & $\mathbf{7.76{\times}10^{-5}}$ \\
\bottomrule
\end{tabular}
\end{table}


\paragraph{Input encoding comparison.}
\label{app:fiete_encoding}
Table~\ref{tab:fiete_encoding_full} shows all encoding${\times}$model
combinations. Fourier features benefit all models, with VN-1L improving
$5{\times}$ and the SAE improving $10{\times}$ over bump.
The $\beta$-VAE benefits most from Fourier-14 ($2.4{\times}10^{-4}$).
The SAE shows dramatic encoding sensitivity (from $1.8{\times}10^{-2}$
to $4.2{\times}10^{-4}$), while VN-3L is robust across all encodings
($5.5{-}7.4{\times}10^{-5}$).

\begin{table}[h]
\centering
\caption{\textbf{Full encoding comparison} ($N{=}2$ \ood{}-1 MSE).}
\label{tab:fiete_encoding_full}
\smallskip
\small
\begin{tabular}{@{}lccccc@{}}
\toprule
Encoding (dim) & SAE & $\beta$-VAE ($\beta{=}1$) & $\beta$-VAE ($\beta{=}4$) & VN-1L & VN-3L \\
\midrule
Bump (56) & $4.37{\times}10^{-3}$ & $4.75{\times}10^{-4}$ & $8.22{\times}10^{-4}$ & $2.93{\times}10^{-4}$ & $6.90{\times}10^{-5}$ \\
One-hot (56) & $4.76{\times}10^{-3}$ & $5.52{\times}10^{-4}$ & $5.95{\times}10^{-4}$ & $3.02{\times}10^{-4}$ & $7.43{\times}10^{-5}$ \\
Fourier-14 (56) & $4.24{\times}10^{-4}$ & $2.38{\times}10^{-4}$ & $4.88{\times}10^{-4}$ & $6.29{\times}10^{-5}$ & $5.49{\times}10^{-5}$ \\
Fourier-7 (28) & $3.60{\times}10^{-3}$ & $2.69{\times}10^{-4}$ & $3.36{\times}10^{-4}$ & $5.93{\times}10^{-5}$ & $5.90{\times}10^{-5}$ \\
Scalar (2) & $1.81{\times}10^{-2}$ & $1.42{\times}10^{-3}$ & $1.72{\times}10^{-3}$ & $1.04{\times}10^{-3}$ & $6.56{\times}10^{-5}$ \\
\bottomrule
\end{tabular}
\end{table}


\paragraph{Sharper bumps ($\sigma{=}0.5$).}
Table~\ref{tab:fiete_sharp}: the Full~2D~VN holds at $1.08{\times}$
ratio; on $N{=}2$ sharp bumps it is $7.5{\times}$ more precise than
ML-Fact.\ VN.

\begin{table}[h]
\centering
\caption{\textbf{$\sigma{=}0.5$}: single-bump (left) and $N{=}2$ (right).}
\label{tab:fiete_sharp}
\smallskip
\small
\begin{tabular}{@{}lccc@{}}
\toprule
Model & \ood{}/\id{} & $N{=}2$ \ood{}-1 \\
\midrule
CNN & $3.39{\times}$ & $3.09{\times}10^{-3}$ \\
Fiete Fact. & $1.05{\times}$ & $4.29{\times}10^{-3}$ \\
ML-Fact.\ VN & $0.97{\times}$ & $2.08{\times}10^{-3}$ \\
Full 2D VN & $1.08{\times}$ & $\mathbf{2.71{\times}10^{-4}}$ \\
\bottomrule
\end{tabular}
\end{table}

\paragraph{Annular holdout.}
Ring-shaped \ood{}: Full~2D~VN $1.04{\times}$, CNN $4.57{\times}$.

\paragraph{Emergent factorization.}
\label{app:fiete_gaps}
Removing hand-built components step by step
(supervised $S$ $\to$ Hebbian $\to$ no axis loss $\to$ flat 2D dict)
preserves compositional generalization at every stage
($1.31{\times}$ $\to$ $1.56{\times}$ $\to$ $1.47{\times}$ $\to$ $1.19{\times}$).

\paragraph{Entangled input.}
\label{app:entangled}
When input${=}$output (784-D flattened image), all models solve $N{=}2$
trivially (${\sim}10^{-5}$) because composition reduces to autoencoding.
The challenge requires a dimensionality mismatch (56-D $\to$ 784-D).

\section{Function Composition: Extended Results}
\label{app:fc}

This section extends the function-composition results from Fig.~\ref{fig:fc_overview} with full metric grids, qualitative Hard-OOD examples, and masked-inference details.

This appendix contains two blocks of supporting material:
(i) the full evaluation grids behind both protocols -- reconstruction
(Appendix~\ref{app:fc-full-grid}) and masked prediction
(Appendix~\ref{app:fc-mask-grid});
(ii) qualitative-reconstruction galleries for both protocols
(Appendix~\ref{app:fc-galleries}) and algorithmic details of
masked-FISTA inference (Appendix~\ref{app:fc-masked-fista}).

\subsection{Full Reconstruction Grid (Clean Training)}
\label{app:fc-full-grid}

Table~\ref{tab:fc-full-grid} reports the per-model, per-difficulty
reconstruction MSE underlying the bars summarized in
Fig.~\ref{fig:fc_overview}b. The masked-prediction counterpart
appears in Table~\ref{tab:fc-mask-grid} and qualitative galleries for
both protocols are collected in App.~\ref{app:fc-galleries}.

\begin{table}[h]
  \centering
  \scriptsize
  \setlength{\tabcolsep}{4pt}
  \renewcommand{\arraystretch}{1.05}
  \begin{tabular}{l
                  S[table-format=1.2e-1] S[table-format=1.2e-1] S[table-format=1.2e-1]}
    \toprule
    Model & {ID} & {Easy} & {Hard} \\
    \midrule
    NeuralODE              & 2.16e-1 & 1.20e+0 & 1.44e+0 \\
    rSLDS                  & 1.82e-1 & 1.12e+0 & 1.18e+0 \\
    DeepAE3                & 1.90e-4 & 3.37e-1 & 4.30e-1 \\
    SparseAE3              & 3.53e-4 & 2.71e-1 & 3.71e-1 \\
    RNN                    & 5.61e-4 & 1.72e-2 & 5.34e-2 \\
    VN-1L                  & 2.76e-2 & 4.26e-2 & 7.18e-2 \\
    Transformer            & 2.1e-5  & 2.7e-4  & 1.5e-3 \\
    Mamba-bottleneck       & 4.60e-5 & 1.69e-2 & 1.89e-2 \\
    VN-5L                  & 7.80e-5 & 1.59e-4 & 4.21e-3 \\
    VN-3L                  & 9.60e-5 & 1.73e-4 & 4.16e-3 \\
    Mamba                  & 1.60e-5 & 3.05e-4 & 8.22e-4 \\
    \bottomrule
  \end{tabular}
  \caption{\textbf{Full Function-Composition reconstruction grid}
    (clean training; per-model, per-difficulty MSE).}
  \label{tab:fc-full-grid}
\end{table}

\subsection{Full Mask-Prediction Grid (Clean Training)}
\label{app:fc-mask-grid}

Table~\ref{tab:fc-mask-grid} reports the mean masked MSE for every
(model, mask regime, difficulty) cell at Tier~1 ($n{=}600$ test
samples, clean training).  All VN numbers use masked-FISTA inference
with $n_{\mathrm{outer}}{=}5$; all other models receive the masked
input with zeroed positions.  The
\textsc{mask\_forecast\_25}/Hard-OOD column corresponds to the
masked-prediction bars summarized in
Fig.~\ref{fig:fc_overview}c.

\begin{table}[h]
  \centering
  \scriptsize
  \setlength{\tabcolsep}{3.0pt}
  \renewcommand{\arraystretch}{1.05}
  \begin{tabular}{l
                  S[table-format=1.2e-1] S[table-format=1.2e-1] S[table-format=1.2e-1]
                  S[table-format=1.2e-1] S[table-format=1.2e-1] S[table-format=1.2e-1]}
    \toprule
    & \multicolumn{3}{c}{\textsc{mask\_forecast\_25}}
    & \multicolumn{3}{c}{\textsc{mask\_forecast\_50}} \\
    \cmidrule(lr){2-4}\cmidrule(lr){5-7}
    Model
      & {ID} & {Easy} & {Hard}
      & {ID} & {Easy} & {Hard} \\
    \midrule
    \textbf{VN-3L}      & 5.50e-1 & 8.06e-1 & \textbf{3.91e-1} & 8.84e-1 & 9.38e-1 & \textbf{6.22e-1} \\
    \textbf{VN-5L}      & 5.50e-1 & 8.04e-1 & \textbf{3.92e-1} & 8.85e-1 & 9.37e-1 & \textbf{6.22e-1} \\
    VN-1L               & 3.15e-1 & 5.90e-1 & 5.81e-1 & 8.24e-1 & 8.86e-1 & 7.73e-1 \\
    Mamba-bottleneck    & 5.40e-1 & 8.72e-1 & 6.24e-1 & 7.82e-1 & 9.80e-1 & 7.36e-1 \\
    RNN                 & 6.88e-1 & 8.25e-1 & 1.23e+0 & 8.08e-1 & 9.50e-1 & 8.21e-1 \\
    DeepAE3             & 1.76e-1 & 7.34e-1 & 1.49e+0 & 5.22e-1 & 1.01e+0 & 9.46e-1 \\
    SparseAE3           & 1.80e-1 & 6.73e-1 & 1.51e+0 & 5.37e-1 & 9.87e-1 & 8.95e-1 \\
    Transformer         & 1.22e+0 & 9.68e-1 & 1.48e+0 & 1.07e+0 & 1.00e+0 & 9.57e-1 \\
    Mamba               & 1.34e+0 & 9.80e-1 & 1.76e+0 & 1.11e+0 & 1.01e+0 & 1.14e+0 \\
    NeuralODE           & 1.92e+0 & 1.61e+0 & 1.27e+0 & 2.16e+0 & 1.91e+0 & 1.25e+0 \\
    rSLDS               & 1.92e+0 & 1.56e+0 & 2.34e+0 & 2.11e+0 & 2.17e+0 & 2.15e+0 \\
    \midrule
    & \multicolumn{3}{c}{\textsc{mask\_random\_30}}
    & \multicolumn{3}{c}{\textsc{mask\_block\_128}} \\
    \cmidrule(lr){2-4}\cmidrule(lr){5-7}
    Model
      & {ID} & {Easy} & {Hard}
      & {ID} & {Easy} & {Hard} \\
    \midrule
    \textbf{VN-3L}      & 6.61e-3 & 5.71e-3 & 1.31e-2 & 4.59e-1 & 6.92e-1 & \textbf{2.97e-1} \\
    \textbf{VN-5L}      & 5.56e-3 & 4.58e-3 & \textbf{1.20e-2} & 4.42e-1 & 6.79e-1 & 3.04e-1 \\
    VN-1L               & 8.49e-2 & 8.36e-2 & 1.71e-1 & 2.84e-1 & 4.10e-1 & 5.12e-1 \\
    Mamba-bottleneck    & 1.83e-2 & 4.20e-2 & 4.35e-2 & 5.36e-1 & 8.21e-1 & 4.82e-1 \\
    RNN                 & 9.65e-2 & 1.36e-1 & 2.21e-1 & 5.06e-1 & 7.82e-1 & 6.91e-1 \\
    DeepAE3             & 5.41e-2 & 3.85e-1 & 5.71e-1 & 2.02e-1 & 7.23e-1 & 7.90e-1 \\
    SparseAE3           & 5.44e-2 & 3.81e-1 & 5.25e-1 & 1.97e-1 & 6.96e-1 & 7.58e-1 \\
    Transformer         & 2.46e-1 & 2.36e-1 & 2.33e-1 & 8.12e-1 & 9.59e-1 & 7.62e-1 \\
    Mamba               & 2.35e-1 & 2.33e-1 & 2.40e-1 & 9.00e-1 & 9.95e-1 & 8.18e-1 \\
    NeuralODE           & 1.00e+0 & 1.43e+0 & 1.41e+0 & 9.65e-1 & 1.60e+0 & 1.92e+0 \\
    rSLDS               & 6.74e-1 & 1.40e+0 & 1.11e+0 & 7.99e-1 & 1.66e+0 & 2.12e+0 \\
    \bottomrule
  \end{tabular}
    \caption{\textbf{Full Function-Composition mask-prediction grid}
    (mean masked MSE, clean training, $n{=}600$).  Rows ordered by the
    main-paper ranking on \textsc{mask\_forecast\_25}/Hard.  Bold:
    best entry per (regime, difficulty) column in the Hard-OOD block.
    VN is the top-2 on every Hard-OOD cell; on
    \textsc{mask\_random\_30}/Hard the margin is roughly two orders of
    magnitude over the autoencoder family.}
  \label{tab:fc-mask-grid}
\end{table}

\subsection{Hard-OOD Qualitative Galleries}
\label{app:fc-galleries}

Aggregated MSE numbers in Tables~\ref{tab:fc-full-grid}
and~\ref{tab:fc-mask-grid} establish that the VN stacks lead the
Hard-OOD block, but they do not reveal \emph{how} the non-VN
baselines fail.  Two qualitatively different failure modes can
produce similar mean MSEs (e.g.\ uniform shrinkage to the dataset
mean vs.\ tracking the gross shape but losing high-frequency
content), and these have different implications for what the model
has actually learned.  The galleries below inspect representative
worst-case non-VN samples to surface per-baseline failure
signatures.

Figures~\ref{fig:fc-hard-examples}--\ref{fig:fc-hard-examples-gallery}
show the reconstruction-protocol qualitative behavior.  The
single-example figure is the Hard-OOD sample on which the gap between
VN and the three non-VN principled baselines
(\textsc{Mamba-bottleneck}, \textsc{DeepAE3}, \textsc{RNN}) is largest
under clean reconstruction; the gallery extends this to four samples
over the full 10-model roster.  Across all four,
\textsc{DeepAE3} collapses to a mean / linear interpolant, the RNN
traces a low-frequency envelope, and \textsc{Mamba-bottleneck}
follows the gross signal shape but misses the high-frequency detail,
while both VN stacks track within the visible line width.

\begin{figure}[h]
  \centering
  \includegraphics[
    width=0.98\linewidth,
    trim={0pt 0pt 0pt 32pt},  
    clip
  ]{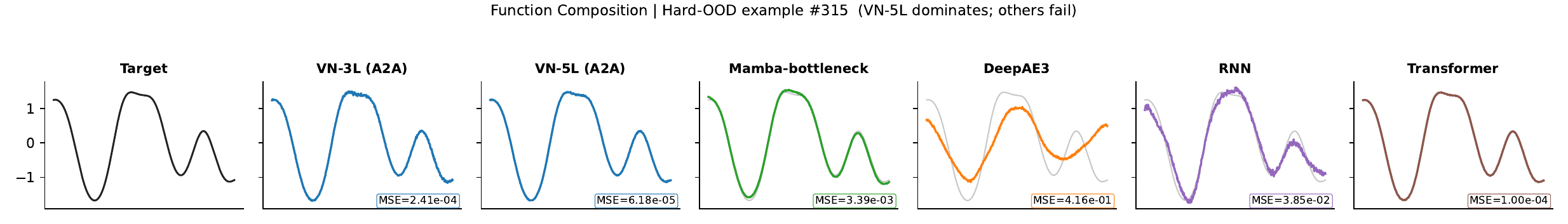}
  \caption{\textbf{Hard-OOD reconstruction (single worst non-VN
  sample)}.  Sample selected as $\arg\max_{i}\rho_{i}$ with
  $\rho_{i}{=}\mathrm{mean}(\mathrm{MSE}_{\mathrm{non-VN}})/\mathrm{MSE}_{\textsc{VN-5L}}$.}
  \label{fig:fc-hard-examples}
\end{figure}

\begin{figure*}[h]
  \centering
  \includegraphics[
    width=0.98\linewidth,
    trim={0pt 0pt 0pt 32pt},  
    clip
  ]{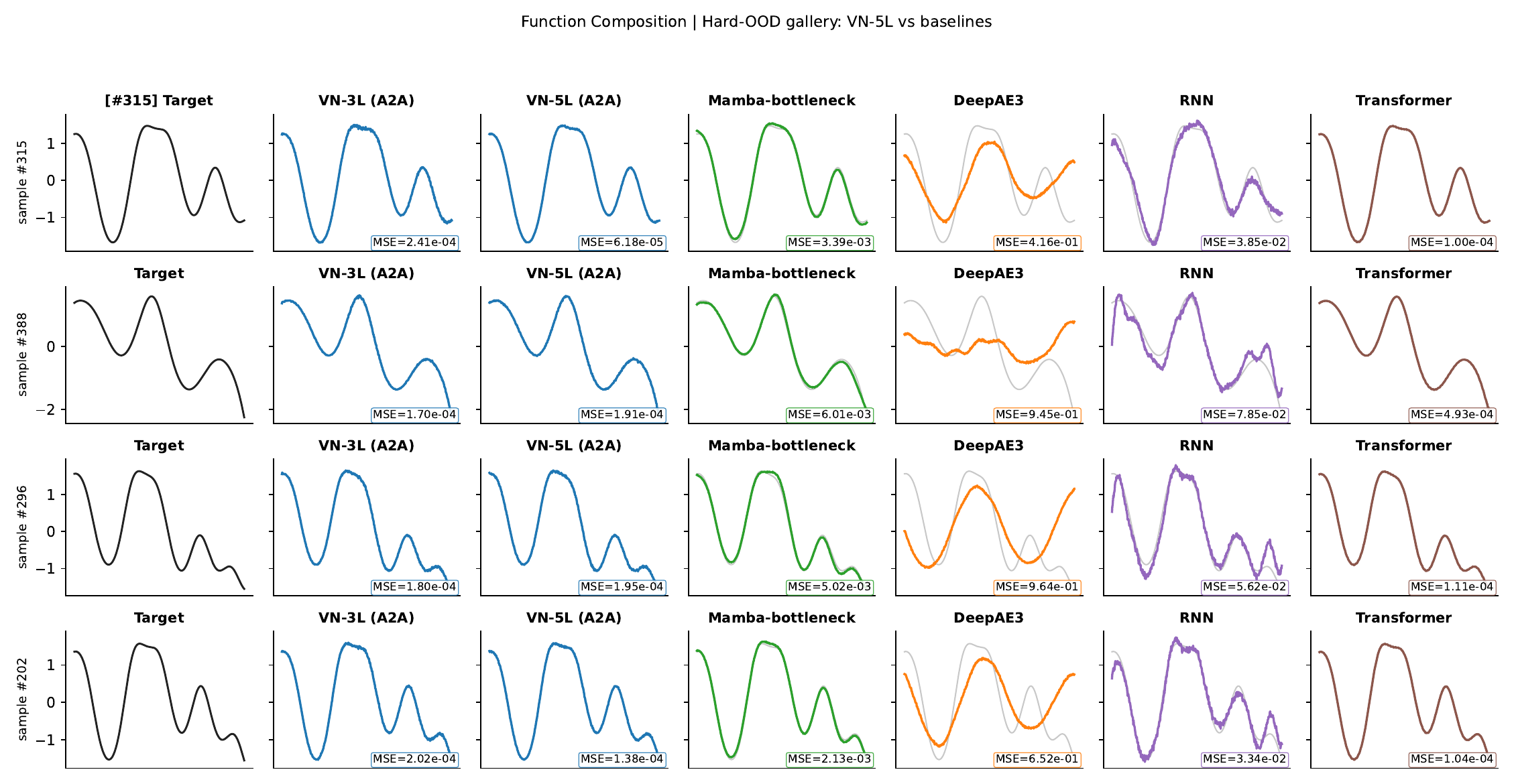}
  \caption{\textbf{Hard-OOD reconstruction gallery} (four worst-case
  non-VN samples).}
  \label{fig:fc-hard-examples-gallery}
\end{figure*}

Figure~\ref{fig:fc-hard-examples-gallery-mask} reports the mask-protocol
counterpart of the reconstruction gallery.  Three Hard-OOD samples
are chosen by
$\rho^{\mathrm{mask}}_{i}=
  \mathrm{mean}(\mathrm{MSE}^{\mathrm{mask}}_{\mathrm{non-VN}})/
  \mathrm{MSE}^{\mathrm{mask}}_{\textsc{VN-5L}}$
on the hidden positions and the same mild ID-outlier filter as in the
main paper.  On every row both VN stacks track the signal closely
through the masked tail; the autoencoder family collapses, latent-
dynamics models drift off the signal entirely, and both Mamba variants
catch the early observed part but diverge once the mask hides the
tail.

\begin{figure*}[h]
  \centering
  \includegraphics[
    width=0.98\linewidth,
    trim={0pt 0pt 0pt 20pt},  
    clip
  ]{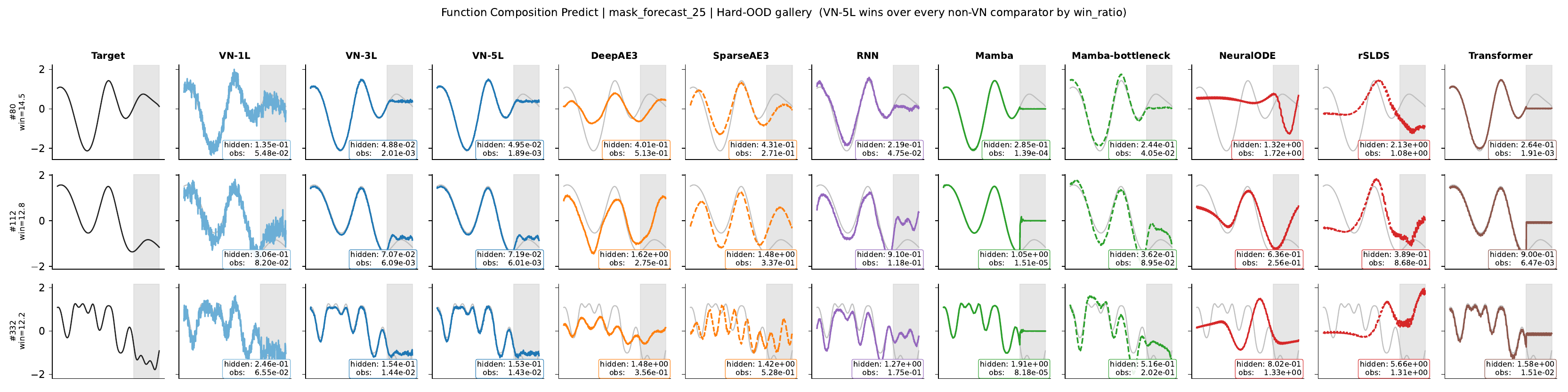}
  \caption{\textbf{Hard-OOD masked-prediction gallery}
  (\textsc{mask\_forecast\_25}).  Three worst-case non-VN samples over
  the full ten-comparator roster.  Grey curves: target; grey shading:
  hidden region; colored curves: model output; per-panel numbers:
  MSE on hidden positions and on observed positions.}
  \label{fig:fc-hard-examples-gallery-mask}
\end{figure*}

The non-VN failures are not a single uniform phenomenon: each baseline
class has its own qualitative signature
(\textsc{DeepAE3} mean collapse, RNN low-frequency envelope,
\textsc{Mamba-bottleneck} shape-without-detail; under masking,
autoencoders collapse on the hidden tail and Mamba variants diverge
once the observation window ends).  The aggregate MSE gap therefore
reflects different underlying inductive biases rather than a uniform
quantitative gap.  Both VN stacks consistently track the signal within
the visible line width across both protocols, which matches the
quantitative ranking in
Tables~\ref{tab:fc-full-grid}--\ref{tab:fc-mask-grid}.

\subsection{Masked-FISTA VN Inference}
\label{app:fc-masked-fista}

For an observed signal $x_{\mathrm{obs}}\in\R^D$, mask $m\in\{0,1\}^D$, and frozen
VN dictionary $S$, the masked imputation problem is
\begin{equation}
\label{eq:fc-masked-argmin}
\min_{z}\ \tfrac{1}{2}\bigl\|m \odot (x_{\mathrm{obs}} - S z)\bigr\|_2^2 \;+\; \lambda\,\|z\|_1,
\end{equation}
where the observation-mismatch loss is restricted to the observed
positions and hidden positions are filled by the sparse linear
synthesis $S z$.

Masked-FISTA VN is an outer imputation loop around the ordinary VN
inference call $\textsc{VN\_infer}(\cdot)$ that solves
Eq.~\eqref{eq:fc-masked-argmin} without retraining any VN weights
(Algorithm~\ref{alg:masked_fista}).

\begin{algorithm}[H]
\caption{Masked-FISTA VN inference}
\label{alg:masked_fista}
\begin{algorithmic}[1]
\Require Observed signal $x_{\mathrm{obs}}\in\R^{D}$, mask $m\in\{0,1\}^{D}$, frozen VN stack with dictionary $S$, outer iterations $n_{\mathrm{outer}}{=}5$
\State $z^{(0)} \leftarrow \textsc{VN\_infer}(m\odot x_{\mathrm{obs}})$
\For{$k = 1,\dots,n_{\mathrm{outer}}$}
  \State $x^{(k)} \leftarrow m\odot x_{\mathrm{obs}} + (1-m)\odot S\,z^{(k-1)}$
  \State $z^{(k)} \leftarrow \textsc{VN\_infer}(x^{(k)})$
\EndFor
\State \Return $\hat{x} = S\,z^{(n_{\mathrm{outer}})}$
\end{algorithmic}
\end{algorithm}

At convergence, observed positions match $x_{\mathrm{obs}}$ exactly and
hidden positions are filled by the sparse linear synthesis
$S\,z^{\star}$.  The fixed point coincides with the minimiser of
\eqref{eq:fc-masked-argmin} because each inner call already solves an
$\ell_{1}$-regularized residual against a fully observed signal.  We
fix $n_{\mathrm{outer}}{=}5$ throughout the paper; empirically five
outer steps are sufficient for the masked residual to fall below the
FISTA tolerance on every VN model we tested.  No VN weight is updated
in this procedure, so the method transfers the reconstruction-trained
dictionary to the prediction setting at a $\approx 5\times$ inference
cost and \emph{zero} retraining cost.  Non-VN baselines are evaluated
with hidden positions set to zero, exactly matching how their
recurrent / state-space / autoencoder decoders would see masked
evidence in deployment.

\section{n-Body Dynamics: Extended Results}
\label{app:nbody_extended}

This section extends the n-body results from Fig.~\ref{fig:gravity_comp} with force-law details, rollout analyses, and additional OOD probes. These diagnostics test whether the same compositional advantages seen in one-step prediction remain visible under longer-horizon physical dynamics.

This appendix consolidates the n-body force-law details:
force-law definitions and design rationale
(Sec.~\ref{app:force_laws}), the top-down coupling ablation
(Sec.~\ref{app:td_ablation}) and trajectory rollouts
(Sec.~\ref{app:rollouts}).

\subsection{Force Law Definitions}
\label{app:force_laws}

All simulations use five unit-mass particles in a 2D box of side
length~5.  Initial positions are drawn uniformly from
$[-2.5, 2.5]^2$; initial velocities from
$\mathcal{N}(0, 0.5^2)$.

The four force primitives acting on particle~$i$ due to particle~$j$
are:
\begin{equation}
\label{eq:gravity}
\mathbf{F}_\text{grav}
= -G \frac{\hat{\mathbf{r}}_{ij}}{(r_{ij}^2 + \epsilon^2)},
\qquad
G = 1.0,\; \epsilon = 0.1
\end{equation}
\begin{equation}
\label{eq:spring}
\mathbf{F}_\text{spring}
= -k\,(r_{ij} - r_0)\,\hat{\mathbf{r}}_{ij},
\qquad
k = 0.5,\; r_0 = 1.0
\end{equation}
\begin{equation}
\label{eq:drag}
\mathbf{F}_\text{drag}
= -\gamma\,\mathbf{v}_i,
\qquad
\gamma = 0.3
\end{equation}
\begin{equation}
\label{eq:lorentz}
\mathbf{F}_\text{lor}
= \omega \,\bigl(-v_{i,y},\, v_{i,x}\bigr),
\qquad
\omega = 2.0
\end{equation}
where $r_{ij} = \|\mathbf{r}_i - \mathbf{r}_j\|$ and
$\hat{\mathbf{r}}_{ij}$ is the unit direction from $j$ to $i$.  A
soft boundary dampens particles that drift beyond box size (position
scaled by 0.98, velocity by 0.9).  The integrator is fourth-order
Runge-Kutta with $\Delta t = 0.005$.

\paragraph{Choice of force forms.}
The four primitives were selected so that each carries a distinct
\emph{functional class}, ensuring that the compositional task tests
recombination of qualitatively different interactions rather than
sign- or scale-flipped variants of the same dynamics.

\textit{Gravity (Eq.~\ref{eq:gravity}).}  A softened Newtonian
$1/r^{2}$ attraction between unit masses, the canonical long-range
inverse-square force.  The softening $\epsilon\!=\!0.1$ removes the
$r\!\to\!0$ singularity without altering the long-range form.

\textit{Spring (Eq.~\ref{eq:spring}).}  Hooke's law with rest length
$r_{0}\!=\!1$: linear in displacement, equilibrium-seeking,
qualitatively unlike any inverse-power law, and the only primitive
whose sign flips at a finite separation.

\textit{Drag (Eq.~\ref{eq:drag}).}  Linear viscous damping in the
Stokes regime: a function of the focus body's own velocity rather
than relative position, and the only dissipative primitive.  This is
the one force that breaks energy conservation, deliberately included
to rule out approaches that assume conservative dynamics (e.g.\
Hamiltonian/Lagrangian neural networks).

\textit{Lorentz rotation (Eq.~\ref{eq:lorentz}).}  A velocity-dependent
rotation field of the form \(\omega(-v_{i,y},v_{i,x})\). Unlike the
other three primitives, this force depends directly on the current
velocity rather than relative displacement, making it qualitatively
distinct from gravity, spring, and drag. Its role in the benchmark is
to test whether models can recombine positional and velocity-dependent
mechanisms under unseen compositions without collapsing them into one
shared dense interaction.

\subsection{Body-Dropout Pareto on the n-body Benchmark}
\label{app:body-dropout-pareto}

The main figure (Fig.~\ref{fig:gravity_comp}) reports OOD MSE
aggregated over body counts. Aggregation hides whether a model that
looks competitive at training-time body count $n\!=\!5$ actually
holds up when bodies are dropped to $n\!=\!4,3$, and whether
degradation costs ID, OOD, or both. To expose this we plot, for each
model, the joint trajectory through (Median ID, Median OOD) space as
$n$ shrinks; the resulting Pareto figure makes each baseline's
failure mode geometrically diagnosable.

Figure~\ref{fig:body-dropout-pareto} places each model at
(Median ID MSE, Median OOD MSE) on log--log axes, with three points
per model -- one for each of $n\!=\!5,4,3$ -- connected by a thin
line. Large marker is $n\!=\!5$, small marker $n\!=\!3$. The diagonal
is the no-gap line ($\text{OOD} = \text{ID}$); points below it are
structurally impossible because the OOD pool is a strict superset of
the ID set.

\begin{figure}[h]
  \centering
  \includegraphics[width=0.95\linewidth]{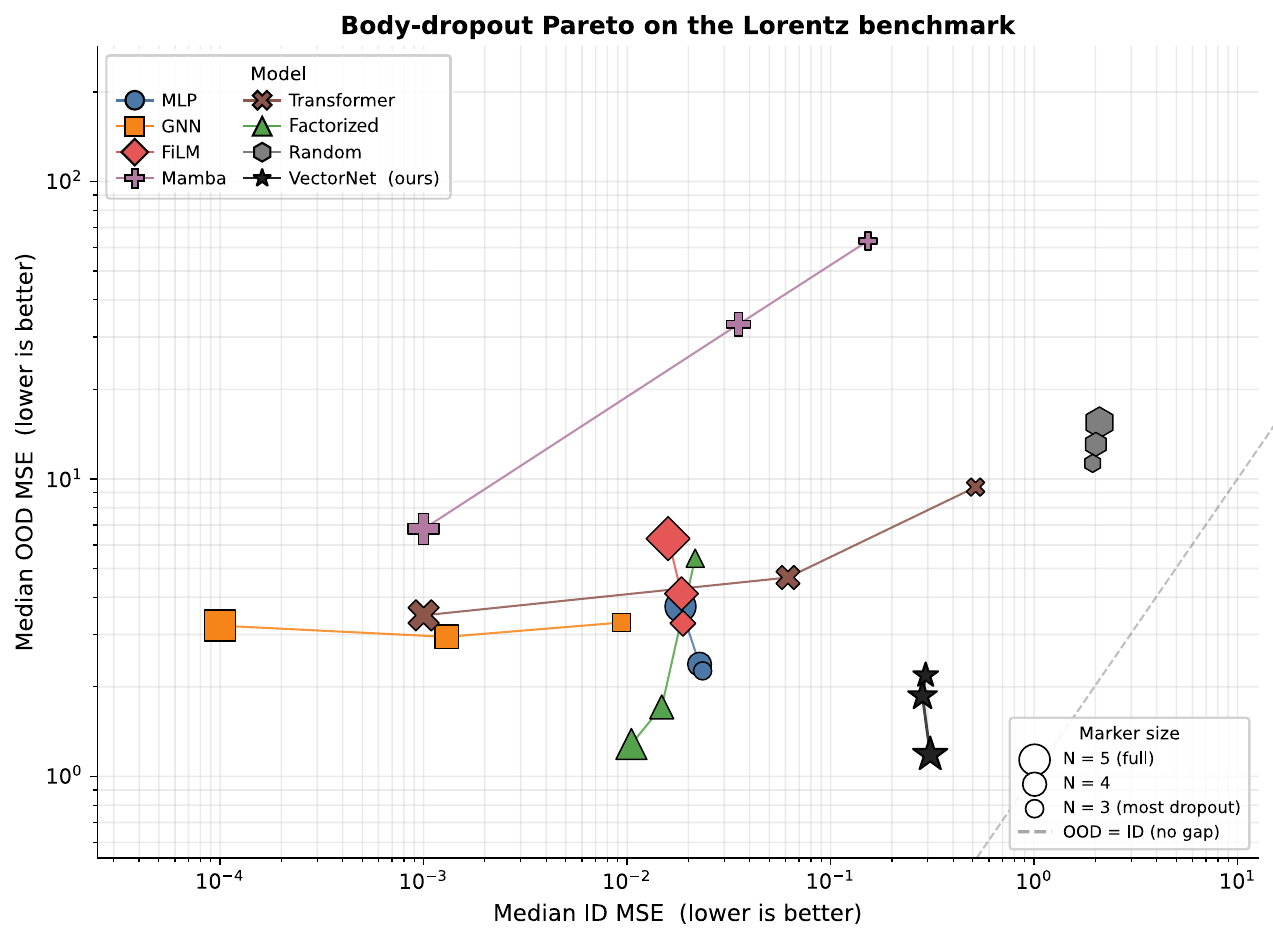}
  \caption{\textbf{Body-dropout Pareto on the n-body benchmark.}
    Median ID vs.\ Median OOD MSE for each model, evaluated at
    $n{=}5,4,3$ bodies (large $\rightarrow$ small marker).}
  \label{fig:body-dropout-pareto}
\end{figure}

Each model traces a characteristic geometric signature:
\textsc{Vector Network} (black star, lower-right) is the only model
that stays in the low-OOD region as $n$ shrinks; its trajectory
barely moves. \textsc{Mamba} (purple) blows up
\emph{vertically} -- an excellent $n\!=\!5$ ID fit but roughly two
orders of magnitude OOD growth as bodies are dropped, the signature
of an architecture that memorizes the training-time problem size.
\textsc{Transformer} (brown) blows up \emph{diagonally}: both ID and
OOD degrade together, ID by ${\sim}500\times$ and OOD by
${\sim}3\times$. \textsc{GNN} (orange) moves \emph{horizontally} --
its ID error grows by ${\sim}100\times$ while OOD barely changes,
the opposite kind of memorization (training body count baked into
the message passing). \textsc{MLP}, \textsc{FiLM}, and
\textsc{Factorized} cluster mid-plot, neither winning ID nor OOD;
\textsc{Random} (grey) fixes the predict-the-mean ceiling and is by
construction insensitive to $n$.

Aggregated MSE hides direction: most baselines that look
``competitive'' at fixed $n$ collapse along a specific axis (ID, OOD,
or both) as $n$ shifts, and each axis corresponds to a recognisable
failure mode. Vector Network is the only model that tolerates body
dropout while remaining in the low-OOD region, evidence that the
learned compositional structure does not bake the training-time
problem size into the operator the network synthesizes per sample.

\subsection{Top-Down Coupling Ablation}
\label{app:td_ablation}

Recipe-development sweeps consistently identified four qualitative
properties of the top-down coupling. First, a small global feedback
weight ($\lambda_{\text{td}}\!\approx\!0.3$) yields a small but
systematic OOD improvement at modest ID cost. Second, this benefit
does not compose with dead-atom re-initialization: the two
mechanisms compete for the same atom-assignment budget. Third,
splitting $\mathbf{W}_{\text{td}}$ per force backfires under the
per-force training budget, since each force-specific projection
sees only a quarter of the data. Fourth, the asymmetry between
top-down and bottom-up coupling is real -- inverting the
direction degrades every metric we tracked. We therefore use
$\lambda_{\text{td}}\!=\!0.3$ as a single global feedback weight in
all reported VN N-body results.

\subsection{Trajectory Rollout Analysis}
\label{app:rollouts}

The N-body benchmark evaluates single-step acceleration prediction,
which leaves open whether good single-step predictions actually
\emph{compose} into physically plausible long-horizon dynamics.
Compositional generalization in this setting is only meaningful if
the per-step predictions add up to stable trajectories, so we
additionally roll each model forward in time and measure how the
error accumulates.

Starting from the true state at $t\!=\!0$, we iterate (i)~construct
the 22-dim per-body input features from the current predicted state,
(ii)~query the model for the instantaneous acceleration
$\hat{\mathbf{a}}_i$ of each body, (iii)~advance velocities and
positions with fourth-order
Runge-Kutta with $\Delta t = 0.005$ described in
App.~\ref{app:force_laws}, and (iv)~apply
the same soft-boundary damping as the simulator (position scaled by
$0.98$, velocity by $0.9$ whenever a body drifts outside the box).
This is an \emph{open-loop} rollout: the model sees its own predicted
state at every step, so errors compound. The rollout covers 999
steps (5\,s of simulated time, $5\times$ longer than a training
sample).

Figure~\ref{fig:drift} shows position MSE vs.\ rollout step on the
ID single-force conditions (top row) and the OOD compositions
(bottom row), with shading showing $\pm 1\sigma$ over four initial
conditions. Vector Network (blue) maintains the lowest drift on five
of six panels, consistent with its strongest single-step accuracy on
the same conditions. The Factorized baseline is the most informative
counter-example: despite very low single-step MSE, it accumulates
position error faster than VN after $\sim\!200$ steps, because its
per-force sub-networks memorise local dynamics but cannot
self-correct once the accumulated state drifts outside the training
distribution. The remaining baselines spread across the panels
without a single model approaching VN's drift envelope.

\begin{figure}[t]
\centering
\includegraphics[width=\linewidth]{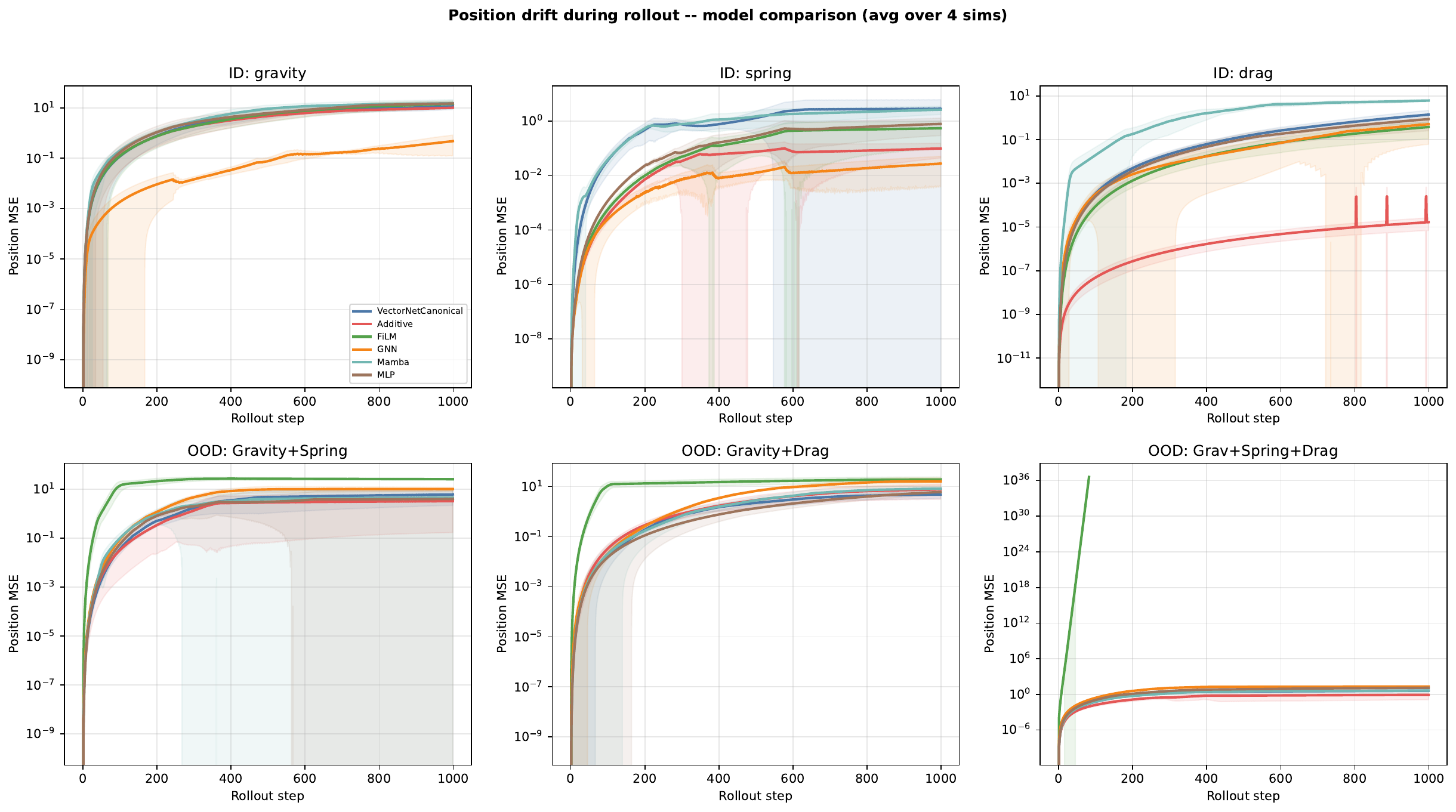}
\caption{%
  \textbf{Position drift during rollout} (5\,s, Runge--Kutta,
  avg.\ over 4 initial conditions, shading $=\,\pm 1\sigma$).  Top
  row: ID single forces; bottom row: OOD compositions.  Vector Network
  (blue) has the lowest average drift on five of six panels.
}
\label{fig:drift}
\end{figure}

Figure~\ref{fig:trajectories} visualizes ground-truth (top row) vs.\
model-predicted body trajectories (subsequent rows) for three OOD
compositions. For each condition we select one simulation to display
from 20 candidates using three filters: (i)~all bodies of the
ground-truth trajectory stay within $0.95\!\cdot\!L_{\text{box}}$ of
the origin (no boundary damping); (ii)~bodies traverse an average of
at least $2$ box units over the rollout (nontrivial dynamics);
and (iii)~among candidates that pass both, we pick the one that
minimizes Vector Network's rollout MSE -- a per-condition best case
for our model that nonetheless uses the same initial state for every
baseline. The figures show that VN trajectories largely overlap
ground truth, while baseline trajectories diverge in
characteristically different ways (identity-style straight lines,
boundary zig-zag, FiLM divergence).

\begin{figure}[t]
\centering
\includegraphics[width=0.88\linewidth]{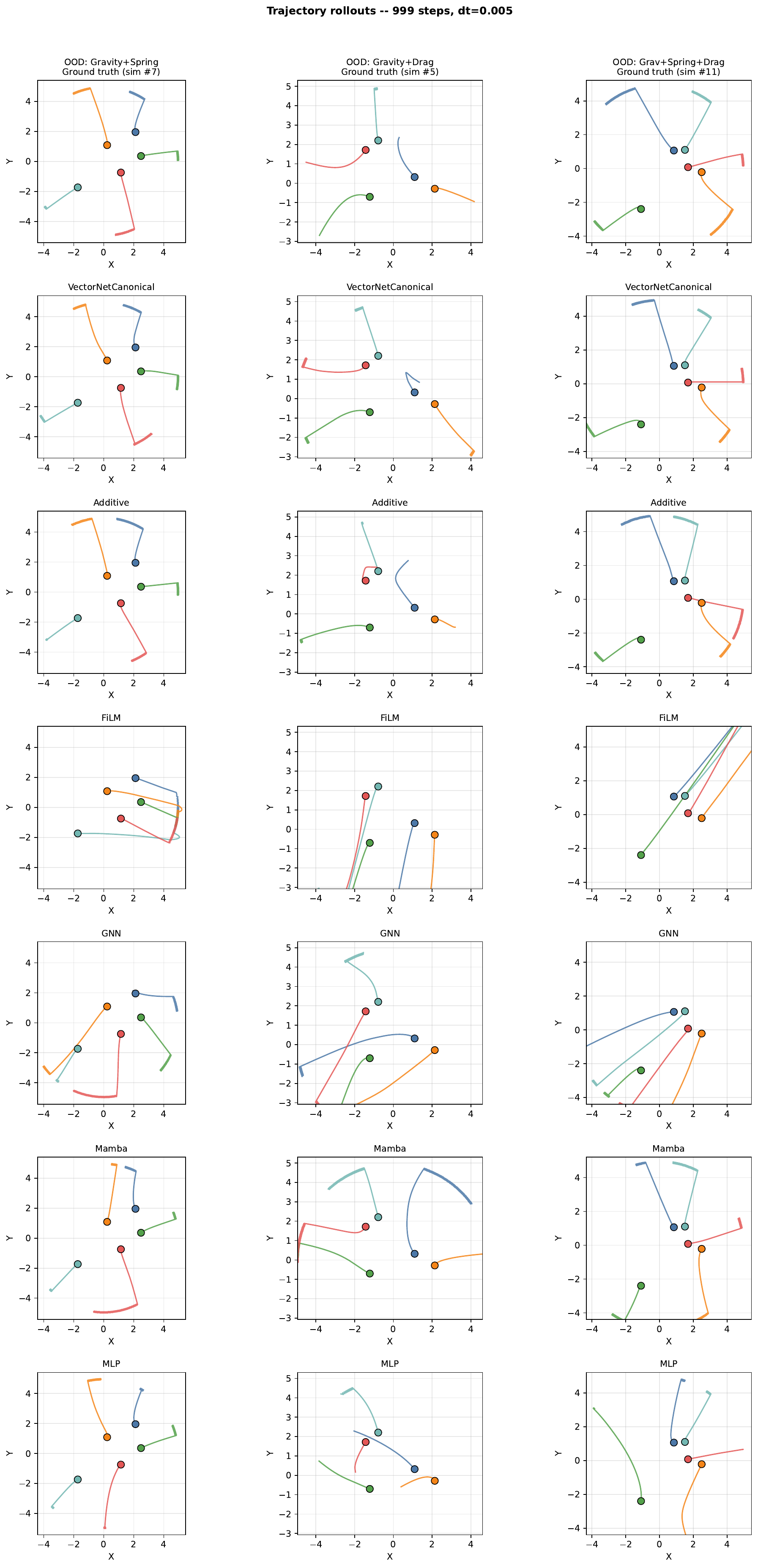}
\caption{%
  \textbf{Trajectory rollouts} under three OOD force compositions
  (columns).  Top row: ground truth; subsequent rows: model rollouts
  from the same initial state per column.  Axes are matched within
  each column.  Dots mark the $t\!=\!0$ positions.
}
\label{fig:trajectories}
\end{figure}

Single-step accuracy translates into long-horizon stability for
Vector Network, as shown by the lowest drift on five of six drift
panels and the closest qualitative match to ground truth on every
OOD trajectory. The visible rollout artefacts (straight-line
segments, soft-boundary zig-zag, FiLM divergence) reflect open-loop
integration of single-step predictors and do not alter the relative
ordering of models in Fig.~\ref{fig:drift}; the trajectory figures
are therefore best read as a qualitative complement to the drift
curves rather than as a standalone evaluation. Crucially, the
single-step \emph{vs.}\ rollout comparison shows that low single-step
MSE alone (e.g.\ Factorized) is not sufficient for compositional
stability over long horizons.

\section{MNIST Composition: Extended Results}
\label{app:mnist_extended}

This section extends the MNIST results from Fig.~\ref{fig:mnist_rts} with representation, sparsity, and convergence diagnostics. The goal is to show how the learned sparse codes use representational budget and how stable the inferred supports become during fast settling.

Figure~\ref{fig:mnist_extended} provides additional diagnostics for the
MNIST experiments beyond the main-paper OOD reconstruction errors.
These panels quantify how the learned sparse codes use representational
budget, how quickly the inferred supports stabilize during fast
settling, and how sparsity changes across depth. Together they show
that the MNIST VNs do not merely achieve low OOD error, but do so with
progressively more selective deeper codes whose active subsets lock in
reliably during inference.

\begin{figure}[H]
  \centering
  \IfFileExists{MNIST_figures/MNIST_extended.png}{%
    \includegraphics[width=\textwidth]{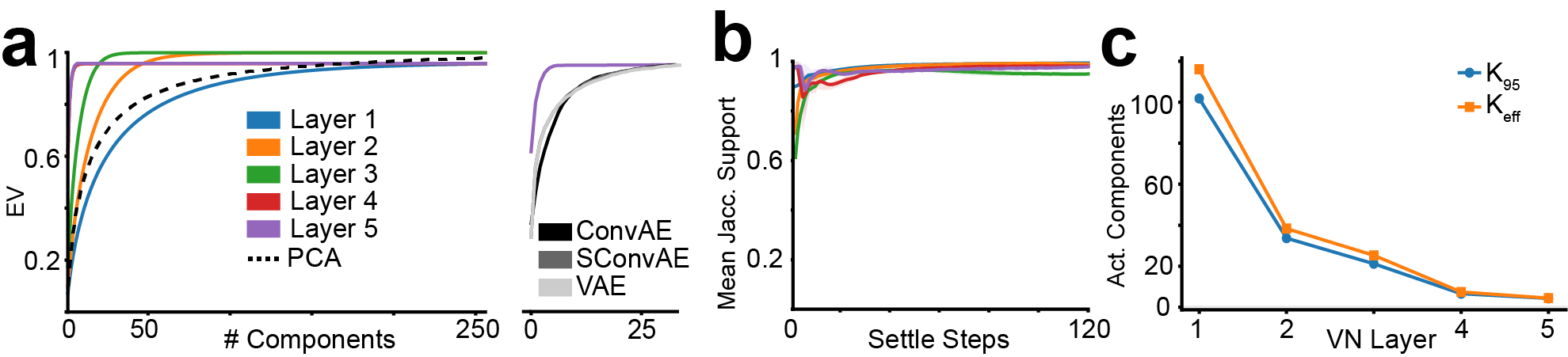}%
  }{%
    \fbox{\rule{0pt}{0.38\textheight}\rule{0.98\textwidth}{0pt}}%
  }
  \caption{\textbf{Extended MNIST diagnostics.}
  \textbf{(a)} Explained-variance versus representational-budget
  curves. The left panel compares VN layers against PCA by plotting the
  fraction of variance explained as the number of active atoms or
  components increases, showing how deeper VN layers trade denser
  reconstructions for progressively more selective sparse codes. The
  right panel compares the top-layer VN against autoencoder baselines
  at matched representational budgets, making explicit that the VN
  reaches strong reconstruction quality with fewer active components
  than dense latent baselines.
  \textbf{(b)} Support stability across 120 fast-settle sweeps,
  measured by overlap between successive sparse supports. This panel
  shows how quickly the active atom set locks in during recurrent
  inference: early sweeps allow support reconfiguration, after which
  the inferred sparse code becomes stable while the residual continues
  to refine.
  \textbf{(c)} Layerwise sparsity statistics, reported as
  \(k_{95}\) and \(k_{\mathrm{eff}}\). Here \(k_{95}\) is the number
  of active atoms needed to recover \(95\%\) of the final explained
  variance, while \(k_{\mathrm{eff}}\) summarizes the effective support
  size of the inferred code. The trend across layers shows the intended
  compression cascade: higher layers operate with tighter active
  subsets and more selective atom usage than lower layers.}
  \label{fig:mnist_extended}
\end{figure}

\section{Hardware and Storage Requirements}
\label{app:hardware}

This section summarizes the practical compute and storage requirements for reproducing the experiments in this paper. The intent is to make clear that the reported results do not rely on distributed training or unusually large hardware budgets.

All experiments in this paper were run on consumer-grade single-GPU
workstations.  The original benchmarks were developed and evaluated on
an NVIDIA RTX~2080~Ti (11\,GB VRAM); subsequent runs were reproduced
on a modern NVIDIA RTX~5090, which has substantial memory headroom and
shortens wall-clock training time but does not change any reported
result.  No multi-GPU training, distributed compute, or cluster
infrastructure is required.

System memory of 16\,GB is sufficient for every experiment in the
paper; the machines we used were equipped with 64\,GB, but peak
working-set memory never exceeded the 16\,GB budget.  Disk-storage
requirements are negligible: roughly 1\,GB of free space is enough
to hold the trained-model checkpoints, intermediate activations,
diagnostic JSON bundles, and rendered figures for all experiments
combined.


\end{document}